\documentclass[11pt,reqno]{amsart}              % for a regular run
\usepackage{fullpage,times,graphicx,amssymb,amsmath}
\usepackage{natbib}
 \def\bibfont{\small}%
 \bibpunct[, ]{[}{]}{,}{n}{}{,}%

\usepackage{booktabs}       % professional-quality tables
\usepackage{amsfonts}       % blackboard math symbols
\usepackage{nicefrac}       % compact symbols for 1/2, etc.
\usepackage{graphicx}
\usepackage{subfig}
\usepackage{amssymb}

\usepackage{amsmath}
\usepackage{amssymb}
\usepackage{float}
\usepackage{bbm}
\usepackage{caption}
\usepackage{xcolor}
\usepackage{epstopdf}
\usepackage{algorithm}
\usepackage{algpseudocode}
\ifpdf
  \DeclareGraphicsExtensions{.eps,.pdf,.png,.jpg}
\else
  \DeclareGraphicsExtensions{.eps}
\fi
\usepackage{enumitem}
\setlist[enumerate]{leftmargin=.5in}
\setlist[itemize]{leftmargin=.5in}

\definecolor{ddarkbrown}{rgb}{0.5,0.2,0.05} \definecolor{bbluegray}{rgb}{0.05,0,0.5}

\newtheorem{theorem}{Theorem}[section]
\newtheorem{proposition}[theorem]{Proposition}
\newtheorem{definition}[theorem]{Definition}
\newtheorem{lemma}[theorem]{Lemma}

\renewenvironment{proof}{\textbf{Proof.}}{\QED\bigskip}
\newtheorem{assumption}[theorem]{Assumption}

\newcommand{\BEAS}{\begin{eqnarray*}}
\newcommand{\EEAS}{\end{eqnarray*}}
\newcommand{\BEA}{\begin{eqnarray}}
\newcommand{\EEA}{\end{eqnarray}}
\newcommand{\BEQ}{\begin{equation}}
\newcommand{\EEQ}{\end{equation}}
\newcommand{\BIT}{\begin{itemize}}
\newcommand{\EIT}{\end{itemize}}
\newcommand{\BNUM}{\begin{enumerate}}
\newcommand{\ENUM}{\end{enumerate}}

\newcommand{\BA}{\begin{array}}
\newcommand{\EA}{\end{array}}

\newcommand{\eg}{{\it e.g.}}

\newcommand{\ones}{\mathbf 1}

\newcommand{\reals}{{\mathbb R}}

\newcommand{\Prob}{\mathbb{P}}

\newcommand{\Co}{{\mathop {\bf Co}}}

\newcommand{\QED}{~~\rule[-1pt]{6pt}{6pt}}

\newcommand{\epi}{\mathop{\bf epi}}

\newcommand{\dom}{\mathop{\bf dom}}

\def\fpm{f^{\pm}}
\def\fp{f^{+}}
\def\fm{f^{-}}
\def\fpt{f^{+\top}}
\def\fmt{f^{-\top}}
\def\tpm{\theta^{\pm}}
\def\tp{\theta^{+}}
\def\tm{\theta^{-}}
\def\tpsub{\theta^{+ \text{sub}}}
\def\tmsub{\theta^{- \text{sub}}}
\def\tpmsub{\theta^{\pm \text{sub}}}

\def\tpmoi{{\theta^{\pm}_{\ast_i}}}
\def\tpo{{\theta^{+}_\ast}}
\def\tpoi{{\theta^{+}_{\ast_i}}}
\def\tmo{{\theta^{-}_\ast}}
\def\tmoi{{\theta^{-}_{\ast_i}}}
\def\mupm{\mu^{\pm}}
\def\mup{\mu^{+}}
\def\mum{\mu^{-}}

\def\alphao{\alpha_\ast}
\definecolor{changescolor}{rgb}{1,1,1}

\usepackage[colorlinks=true,breaklinks=true,bookmarks=true,urlcolor=blue,
     citecolor=blue,linkcolor=blue,bookmarksopen=false,draft=false]{hyperref}

\begin{document}
%%%%%%%%%%%%%%%%

%\setcounter{page}{1}
%\DOI{10.1287/xxxx.0000.0000}%


%\RUNTITLE{}

\title{Naive Feature Selection:\\ a Nearly Tight Convex Relaxation for Sparse Naive Bayes.}
% Stable Bounds on the Duality Gap of Finite Sum Minimization Problems.

\author{Armin Askari}
\address{UC Berkeley}
\email{aaskari@berkeley.edu}

\author{Alexandre d'Aspremont}
\address{CNRS \& D.I., UMR 8548,\vskip 0ex
\'Ecole Normale Sup\'erieure, Paris, France.}
\email{aspremon@ens.fr}

\author{Laurent El Ghaoui}
\address{UC Berkeley}
\email{elghaoui@berkeley.edu}

\begin{abstract}
Due to its linear complexity, naive Bayes classification remains an attractive supervised learning method, especially in very large-scale settings. We propose a sparse version of naive Bayes, which can be used for feature selection. This leads to a combinatorial maximum-likelihood problem, for which we provide an exact solution in the case of binary data, or a bound in the multinomial case. We prove that our convex relaxation bounds becomes tight as the marginal contribution of additional features decreases, using a priori duality gap bounds dervied from the Shapley-Folkman theorem. We show how to produce primal solutions satisfying these bounds. Both binary and multinomial sparse models are solvable in time almost linear in problem size, representing a very small extra relative cost compared to the classical naive Bayes. Numerical experiments on text data show that the naive Bayes feature selection method is as statistically effective as state-of-the-art feature selection methods such as recursive feature elimination, $l_1$-penalized logistic regression and LASSO, while being orders of magnitude faster \footnote{A python implementation can be found at \url{https://github.com/aspremon/NaiveFeatureSelection}}.
\end{abstract}

%\KEYWORDS{deterministic inventory theory; infinite linear programming duality; 
%  existence of optimal policies; semi-Markov decision process; cyclic schedule}
%\MSCCLASS{Primary: 90B05; secondary: 90C40, 90C90}
%\ORMSCLASS{Primary: Inventory/production: deterministic multi-item;
%  secondary: dynamic programming/optimal control: deterministic 
%  semi-Markov; programming: infinite dimensional}
%\HISTORY{Received November 20, 2003; revised March 8, 2004, and March 26, 2004.}

\keywords{Sparsity, Non-convex Optimization, Lagrangian Duality, Fenchel Duality, Shapley-Folkman Theorem.}
%\ORMSCLASS{Primary: ; secondary: }

\maketitle

\section{Introduction}
Modern, large-scale data sets call for classification methods that scale mildly (e.g. linearly) with problem size. In this context, the classical naive Bayes model remains a very competitive baseline, due to its linear complexity in the number of training points and features. In fact, it is sometimes the only feasible approach in very large-scale settings, particularly in text applications, where the number of features can easily be in the millions. 

Feature selection, on the other hand, is a key component of machine learning pipelines, for two main reasons: i) to reduce effects of overfitting by eliminating noisy, non-informative features and ii) to provide interpretability. In essence, feature selection is a combinatorial problem, involving the selection of a few features in a potentially large population. State-of-the-art methods for feature selection employ some heuristic to address the combinatorial aspect, and the most effective ones are usually computationally costly. For example, LASSO \cite{tibshirani1996regression} or $l_1$-SVM models \cite{fan2008liblinear} are based on solving a convex problem with an $l_1$-norm penalty in order to achieve sparsity (at the expense of tuning a hyper parameter to attain a desired sparsity level).

%penalty on the vector of regression coefficients, a heuristic to constrain its cardinality which requires tuning a hyper parameter to achieve a desired sparsity level. 

Since naive Bayes corresponds to a linear classification rule, feature selection in this setting is directly related to the sparsity of the vector of classification coefficients. This work is devoted to a sparse variant of naive Bayes. Our main contributions are as follows.
\begin{itemize}
    \item We formulate a sparse naive Bayes problem that involves a direct constraint on the cardinality of the vector of classification coefficients, leading to an interpretable naive Bayes model. No hyper-parameter tuning is required in order to achieve the target cardinality.
    
    \item We derive an exact solution of sparse naive Bayes in the case of binary data, and an approximate upper bound for general data, and show that it becomes increasingly tight as the marginal contribution of features decreases. Primal solutions satisfying these bouds can provably be recovered. Both models can be trained very efficiently, with an algorithm that scales almost linearly with the number of features and data points, just like classical naive Bayes. 
    
    % \item We also consider the case when the input data is uncertain, with adversarial noise models that are realistic in text applications; we develop a robust maximum-likelihood counterpart to the naive Bayes model, leading to simple modifications of the previous algorithm.
    
    \item We show in experiments that our model significantly outperforms simple baselines (\eg, thresholded naive Bayes, odds ratio), and achieves similar performance as more sophisticated feature selection methods, at a fraction of the computing cost.
\end{itemize}

Our tightness results hinge on an application of the Shapley-Folkman theorem to separable optimization problems due to \cite{Aubi76,Ekel99} and a more geometric interpretation of convex relaxations described in e.g. \cite{Lema01}.

\paragraph{Related Work on Naive Bayes Improvements} A large body of literature builds on the traditional naive Bayes classifier. A non-extensive list includes the seminal work by \cite{frank2002locally} introducing Weighted naive Bayes; Lazy Bayesian Learning by \cite{zheng2000lazy}; and the Tree-Augmented naive Bayes method by \cite{friedman1997bayesian}. The paper \cite{webb2005not} improves the computational complexity of the aforementioned methods, while maintaining the same accuracy. For a more complete discussion of modifications to naive Bayes, we refer the reader to \cite{jiang2007survey} and the references therein. 

\paragraph{Related Work on Naive Bayes and Feature Selection} Of particular interest to this work are methods that employ feature selection. \cite{kim2006some} use information-theoretic quantities for feature selection in text classification, while \cite{mladenic1999feature} compare a host of different methods and shows the comparative efficacy of the Odds Ratio method. These methods often use ad hoc scoring functions to rank the importance of the different features. \cite{fleuret2004fast} uses the mutual information to select features in a fast way while \cite{zaidi2013alleviating} employs a weighting approach for selecting relevant features. \cite{boulle2007compression} achieve soft variable selection by introducing bayesian regularization into the training problem.

To our knowledge, the first work to directly address sparsity in the context of naive Bayes, with binary data only, is \cite{zheng2018sparse}. Their model does not directly address the requirement that the weight vector of the classification rule should be sparse, but does identify key features in the process. The method requires solving an approximation to the combinatorial feature selection problem via $l_1$-penalized logistic regression problem with non-negativity constraints, that has the same number of features and data points as the original one. Therefore the complexity of the method is the same as ordinary $l_1$-penalized logistic regression, which is relatively high. In contrast, our binary (Bernoulli) naive Bayes bound is exact, and has complexity almost linear in training problem size.

\section{Background on Naive Bayes}
In this paper, for simplicity only, we consider a two-class classification problem; the extension to the general multi-class case \textcolor{changescolor}{can be done by reducing the problem to multiple binary classification tasks (i.e. a one-vs-all or one-vs-one approach).}

\subsection{Notation} For an integer $m$, $[m]$ is the set $\{1,\ldots,m\}$. The notation $\ones$ denotes a vector of ones, with size inferred from context. The cardinality (number of non-zero elements) in a $m$-vector $x$ is denoted $\|x\|_0$, whereas that of a finite set $\mathcal{I}$ is denoted  $|\mathcal{I}|$. Unless otherwise specified, functional operations (such as $\max(0,\cdot)$) on vectors are performed element-wise. For $k \in [n]$, we say that a vector $w \in \mathbb{R}^n$ is $k$-sparse or has sparsity level $\alpha \%$ if at most $k$ or $\alpha \%$ of its coefficients are nonzero respectively.  For two vectors $f,g \in \reals^m$, $f \circ g \in \reals^m$ denotes the elementwise product. For a vector $z$, the notation $s_{k}(z)$ is the sum of the top $k$ entries. Finally, $\Prob(A)$ denotes the probability of an event $A$.

\subsection{Data Setup}
We are given a non-negative data matrix $X \in \mathbb{R}_+^{n \times m} = [x^{(1)},x^{(2)},\hdots,x^{(n)}]^\top$ consisting of $n$ data points, each with $m$ dimensions (features), and a vector $y \in \{-1,1\}^n$ that encodes the class information for the $n$ data points, with $C_+$ and $C_-$ referring to the positive and negative classes respectively. We define index sets corresponding to each class $C_+,C_-$, and their respective cardinality, and data averages: 
\begin{align*}
 &{\mathcal I}_\pm := \left\{ i \in [n] ~:~ y_i = \pm 1 \right\}, \\
 &n_\pm = |{\mathcal I}_\pm|, \\
&f_\pm := \sum_{i \in {\mathcal I}_\pm} x^{(i)} = \pm (1/2)X^\top (y \pm \ones)   
\end{align*}
\textcolor{changescolor}{For example if $X = [x^{(1)}, \hdots, x^{(10)}]^\top$ and $y \in \{-1,1\}^{10}$ with $y_i = 1$ for $i = 1, \hdots, 5$ and $y_i = -1$ for $i = 6, \hdots, 10$, then we have that $f_+ = \sum_{i=1}^5 x^{(i)}$ and $f_- = \sum_{i=6}^{10} x^{(i)}$.}
% so that $\fp = (1/2)X^\top (y+\ones)$, $\fm = (1/2)X^\top (\ones-y)$.

\subsection{Naive Bayes} We are interested in predicting the class label of a test point $x \in \reals^m$ via the rule $\hat{y}(x) = \arg \max_{\epsilon \in \{-1,1\}} \Prob(C_\epsilon \: | \:  x )$. To calculate the latter posterior probability, we employ Bayes' rule and then use the ``naive'' assumption that features are independent of each other: $\Prob(x \: | \: C_\epsilon) = \prod_{j=1}^m \Prob(x_{j} \: | \: C_\epsilon)$, leading to
\begin{align}\label{eq:nb_test}
    \hat{y}(x) &=  \arg \max_{\epsilon \in \{-1,1\}}\; \log \Prob(C_\epsilon) +  \sum_{j=1}^m \log \Prob(x_{j} | C_\epsilon) .
\end{align}
In \eqref{eq:nb_test}, we need to have an explicit model for $\Prob(x_j|C_i)$; in the case of binary or integer-valued features, we use Bernoulli or categorical distributions, while in the case of real-valued features we can use a Gaussian distribution. We then use the maximum likelihood principle (MLE) to determine the parameters of those distributions. Using a categorical distribution, $\Prob(C_\pm)$ simply becomes the number of data points in $X$ belonging to class $\pm 1$ divided by $n$.

\subsubsection{Bernoulli Naive Bayes}
With binary features, that is, $X \in \{0,1\}^{n \times m}$, we choose the following conditional probability distributions parameterized by two non-negative vectors $\tp, \tm \in [0,1]^m$. For a given vector $x \in \{0,1\}^m$,
\[
    \Prob(x_j ~|~ C_\pm) = (\tpm_j)^{x_j}(1-\tpm_j)^{1-x_j}, \;\; j \in [m],
\]
hence
\[
    \sum_{j=1}^m \log \Prob(x_j ~|~ C_\pm) = x^\top \log \tpm + (\ones - x)^\top \log(\ones - \tpm).
\]
Training a classical Bernoulli naive Bayes model reduces to the problem
\begin{align} \label{eq:bnb_train}
    (\tpo,\tmo) &= \arg\max_{\tp, \tm \in [0,1]^m} \mathcal{L}_{\text{bnb}}(\tp,\tm; X) 
\end{align}
where the loss is a concave function
\begin{align}
    \mathcal{L}_{\text{bnb}}(\tp, \tm) =& \sum_{i \in {\mathcal I}_+} \log \Prob(x^{(i)} ~|~ C_+)  \nonumber 
    +\sum_{i \in {\mathcal I}_-} \log \Prob(x^{(i)} ~|~ C_-) \label{eq:loss-def-bnb}\\
     =& \fpt \log \tp + (n_+ \ones - \fp)^\top \log (\ones - \tp)  \nonumber 
     \\
     + &\fmt \log \tm + (n_-\ones - \fm)^\top \log (\ones - \tm) \nonumber
\end{align}
Note that problem \eqref{eq:bnb_train} is decomposable across features and the optimal solution is simply the MLE estimate, that is, $\tpm_* = \fpm/ n_\pm$. From \eqref{eq:nb_test}, we get a linear classification rule: for a given test point $x \in \reals^m$, we set $\hat{y}(x) = \mbox{\bf sign}(v + w_b^\top x)$, where
\begin{align}
v &:= \log \dfrac{\Prob(C_+)}{\Prob(C_-)} + \ones^\top \Big(\log (\ones - \tpo) - \log(\ones - \tmo)\Big) \nonumber \\
    w_b &:= \log \dfrac{\tpo \circ (\ones - \tmo)}{\tm_\ast\circ (\ones - \tpo)}. 
\end{align}

\subsubsection{Multinomial naive Bayes} With integer-valued features, that is, $X \in \mathbb{N}^{n \times m}$, we choose the following conditional probability distribution, again parameterized by two non-negative $m$-vectors $\tpm \in [0,1]^m$, but now with the constraints $\ones^\top \tpm = 1$: for given $x \in \mathbb{N}^m$, 
\begin{align*}
    \Prob(x ~|~ C_\pm) = \dfrac{(\sum_{j=1}^m x_j)!}{\prod_{j=1}^m x_j!} \prod_{j=1}^m (\tpm_j)^{x_j} ,
\end{align*}
and thus
\begin{align*}
    \log \Prob(x ~|~ C_\pm) = x^\top \log \tpm + \log \left(\dfrac{(\sum_{j=1}^m x_j)!}{\prod_{j=1}^m x_j!} \right) 
\end{align*}
While it is essential that the data be binary in the Bernoulli model seen above, the multinomial one can still be used if $x$ is non-negative real-valued, and not integer-valued. Training the classical multinomial model reduces to the problem
\begin{align} \label{eq:mnb_train}
    (\tp_\ast,\tm_\ast) = &\arg\max_{\tp, \tm \in [0,1]^m} \mathcal{L}_{\text{mnb}}(\tp,\tm) \nonumber \\
    &\ones^\top \tp = \ones^\top\tm = 1
\end{align}
where the loss is again a concave function    
\begin{align}
    \mathcal{L}_{\text{mnb}}(\tp, \tm) =& \sum_{i \in {\mathcal I}_+} \log \Prob(x^{(i)} ~|~ C_+) +
    \sum_{i \in {\mathcal I}_-} \log \Prob(x^{(i)} ~|~ C_-) \nonumber\\
     =& \fpt \log \tp + \fmt \log \tm  \label{eq:loss-def-mnb}
\end{align}
% \begin{equation}\label{eq:loss-def-mnb}
%     \mathcal{L}_{\text{mnb}}(\tp, \tm) := \sum_{i \in {\mathcal I}_+} \log \Prob(x^{(i)} ~|~ C_+) + \sum_{i \in {\mathcal I}_-} \log \Prob(x^{(i)} ~|~ C_-) =  \fpt \log \tp + \fmt \log \tm . \label{eq:loss-def-mnb}
% \end{equation}
Again, problem \eqref{eq:mnb_train} is decomposable across features, with the added complexity of equality constraints on $\tpm$. The optimal solution is the MLE estimate $\tpm_* = {f^\pm}/({\ones^\top f^\pm})$. As before, we get a linear classification rule: for a given test point $x \in \reals^m$, we set $\hat{y}(x) = \mbox{\bf sign}(v + w_m^\top x)$, where
\begin{align}\label{eq:m-rule}
v := \log \Prob(C_+) - \log \Prob(C_-), \;\; w_m := \log \tpo - \log \tmo 
\end{align}

\section{Naive Feature Selection}
In this section, we incorporate sparsity constraints into the aforementioned models.

\subsection{Naive Bayes with Sparsity Constraints} 
\label{subsec:sparse} 
For a given integer $k \in [m]$, with $k <m$,  we seek to obtain a naive Bayes classifier that uses at most $k$ features in its decision rule. 
For this to happen, we need the corresponding coefficient vector, written $w_b$ and $w_m$ above, to be $k$-sparse. For both Bernoulli and multinomial models, this happens if and only if the difference vector $\tpo - \tmo$ is sparse. By enforcing $k$-sparsity on the difference vector, the classifier uses less than $m$ features for classification, making the model more interpretable.

\paragraph{Sparse Bernoulli Naive Bayes}  In the Bernoulli case, the sparsity-constrained problem becomes
\begin{align}\label{eq:bnb0}
    (\tpo,\tmo)  = 
    &\arg\max_{\tp, \tm \in [0,1]^m} \: \mathcal{L}_{\text{bnb}}(\tp,\tm; X) \nonumber \\
    &\|\tp - \tm\|_0 \leq k \tag{SBNB} 
\end{align}
where $\mathcal{L}_{\text{bnb}}$ is defined above. Here, $\|\cdot\|_0$ denotes the $l_0$-norm, or cardinality (number of non-zero entries) of its vector argument, and $k <m$ is the user-defined upper bound on the desired cardinality.

\paragraph{Sparse Multinomial Naive Bayes} In the multinomial case, in light of \eqref{eq:mnb_train}, our model is written
\begin{align}\label{eq:mnb0}
    (\tpo,\tmo)  = 
    &\arg\max_{\tp, \tm \in [0,1]^m} \: \mathcal{L}_{\text{mnb}}(\tp,\tm; X) \nonumber \\ 
    & \ones^\top \tp = \ones^\top\tm = 1 \nonumber \\ 
    &\|\tp - \tm\|_0 \leq k \tag{SMNB} 
\end{align}
where ${\mathcal L}_{\text{mnb}}$ is defined above.

%\subsection{Main Results}
Due to the inherent combinatorial and non-convex nature of the cardinality constraint, and the fact that it couples the variables $\tpm$, the above sparse training problems look much more challenging to solve when compared to their classical counterparts, \eqref{eq:bnb_train} and \eqref{eq:mnb_train}. We will see in what follows that this is not the case.

\subsection{Sparse Bernoulli Case} The sparse counterpart to the Bernoulli model, \eqref{eq:bnb0}, can be solved efficiently in \textit{closed form}, with complexity comparable to that of the classical Bernoulli problem \eqref{eq:bnb_train}. 

\begin{theorem}[Sparse Bernoulli naive Bayes]\label{thm:sparse_bnb}
\; Consider the sparse Bernoulli naive Bayes training problem \eqref{eq:bnb0}, with binary data matrix $X \in \{0,1\}^{n \times m}$. The optimal values of the variables are obtained as follows. Set 
\begin{align}\label{eq:v-w-def-sbnb}
v &:= (\fp+\fm) \circ \log \Big( \dfrac{\fp+\fm}{n}\Big) \\
&+ (n\ones -\fp-\fm) \circ \log \Big(\ones - \dfrac{\fp+\fm}{n}\Big) \nonumber \\
w &:= w^+ + w^-  \\
w^\pm &:= \fpm \circ \log \dfrac{\fp}{n_\pm}+(n_\pm \ones - \fpm) \circ \log \Big(\ones - \dfrac{\fpm}{n_\pm}\Big) . \nonumber
\end{align}
Then identify a set ${\mathcal I}$ of indices with the $k$ largest elements in $w-v$, and set $\tpo,\tmo$ according to
\begin{align}
\tpoi = \tmoi = \frac{1}{n}(\fp_i + \fm_i), \;\forall 
i \in {\mathcal I}, \;\;\;\;
\tpmoi = \dfrac{\fpm_i}{n_\pm} , \; \forall i \not\in {\mathcal I}.
\end{align}
\end{theorem}
\begin{proof}
First note that an $\ell_0$-norm constraint on a $m$-vector $q$ can be reformulated as 
\[
\|q\|_0 \leq k \Longleftrightarrow \exists \: {\mathcal I} \subseteq [m] , \;\; |\mathcal{I}| \leq k  ~:~  \forall \: i \not\in {\mathcal I}, \;\;\; q_i = 0.
\]
Hence problem \eqref{eq:bnb0} is equivalent to
\begin{align}\label{eq:noncvxNB1}
\max_{\tp,\tm \in [0,1]^m, \mathcal{I}}  &\mathcal{L}_{\text{bnb}}(\tp,\tm; X) \nonumber \\  & \text{s.t.} \;\;\tp_i = \tm_i \;\;\forall i \not\in {\mathcal I} , \;\; \mathcal{I} \subseteq [m], \;\; |\mathcal{I}| \leq k
\end{align}
where the complement of the index set ${\mathcal I}$ encodes the indices where variables $\tp, \tm$ agree.
% For $\mu,\nu \in [0,1]^m$, define 
% \[
% \Phi(\mu,\nu) := \nu^\top \log \mu + (1-\nu)^\top \log (1-\mu).
% \]
Then  \eqref{eq:noncvxNB1} becomes
\begin{align}
p^\ast := \max_{\mathcal{I} \subseteq [m], \: |\mathcal{I}| \leq k} &  \;\; \; \Big(\sum_{i \not\in \mathcal{I}} h_i^\pm \Big)
+ \Big(\sum_{i \in \mathcal{I}} h_i^+ + h_i^- \Big)
\end{align}
where 
\begin{align*}
    h_i^\pm &= \max_{\theta_i \in [0,1]} \:
(\fp_i+\fm_i)\log \theta_i + (n-\fp_i-\fm_i) \log (1-\theta_i) \\
    h_i^+ &= \max_{\tp_i\in [0,1]} 
\fp_i \log \tp_i + (n_+ -\fp_i) \log (1 -\tp_i) \\
    h_i^- &= \max_{\tm_i \in [0,1]} \fm_i \log \tm_i+ (n_- - \fm_i) \log(1-\tm_i)
\end{align*}
and where we use the fact that $n_+ + n_- = n$. All the above expressions for $h_i^\pm, h_i^+, h_i^-$ have closed form values and solutions
\begin{align}\label{eq:theta-opt-sbnb}
\theta_i &= \tpo_i = \tmo_i = \frac{1}{n}(\fp_i + \fm_i), \;\;\forall
i \not\in {\mathcal I} \nonumber \\
\tpmoi &= \dfrac{\fpm_i}{n_\pm} , \;\;\forall i \in {\mathcal I}
\end{align}
% For $x \in [0,1]^m$, define the $m$-vector
% \[
% H(x) = x \circ \log x + (\ones -x)\circ \log (\ones -x).
% \]
Plugging the above inside the objective of \eqref{eq:noncvxNB1} results in a Boolean formulation, with a Boolean vector $u$ of cardinality $\le k$ such that $\ones-u$ encodes indices for which entries of $\tp,\tm$ agree:
\begin{align*}
% p^\ast :=& \max_{\mathcal{I} \subseteq [m] , \:  |\mathcal{I}| \leq k} \: 
% 2\sum_{i \not\in \mathcal{I}} H((\fp_i+\fm_i)/2) + \sum_{i \in \mathcal{I}}[H(\fp_i) + H(\fm_i)]
% \\
p^\ast :=& \max_{u \in {\mathcal C}_{k}} \: 
(\ones - u)^\top v + u^\top w ,
\end{align*}
where, for $k \in [m]$:
\[
\mathcal{C}_k := \{ u ~:~  u \in \{0,1\}^m, \; \textbf{1}^\top u  \leq k\}, 
\]
and vectors $v,w$ are as defined in \eqref{eq:v-w-def-sbnb}:
\begin{align*}
v &:= (\fp+\fm) \circ \log \Big( \dfrac{\fp+\fm}{n}\Big) \\
&\;\;\;\; +(n\ones -\fp-\fm) \circ \log \Big(\ones - \dfrac{\fp+\fm}{n}\Big)  \\
w &:= w^+ + w^- \\
w^\pm &:= \fpm \circ \log \dfrac{\fp}{n_\pm} + (n_\pm \ones - \fpm) \circ \log \Big(\ones - \dfrac{\fpm}{n_\pm}\Big)
\end{align*}

We obtain 
\[
p^\ast =  \ones^\top v + \max_{u \in {\mathcal C}_{k}} \: u^\top (w-v) 
= \ones^\top v + s_{k}(w-v),
\]
where $s_{k}(\cdot)$ denotes the sum of the $k$ largest elements in its vector argument. Here we have exploited the fact that the map $z := w-v \ge 0$, which in turn implies that 
\[
s_{k}(z) = \max_{u \in \{0,1\}^m \::\: \textbf{1}^\top u  = k} \: u^\top z =  
\max_{u \in {\mathcal C}_{k}} \: u^\top z.
\]
In order to recover an optimal pair $(\tpo,\tmo)$, we simply identify the set ${\mathcal I}$ of indices with the $m-k$ smallest elements in $w-v$, and set $\tpo,\tmo$ according to \eqref{eq:theta-opt-sbnb}.
\end{proof}

Note that the complexity of the computation (including forming the vectors $\fpm$, and finding the $k$ largest elements in the appropriate $m$-vector) grows as $O(mn + m\log(k))$. This represents a very moderate extra cost compared to the cost of the classical naive Bayes problem, which is $O(mn)$.

\subsection{Multinomial Case}
In the multinomial case, the sparse problem \eqref{eq:mnb0} does not admit a closed-form solution. However, we can obtain an easily computable upper bound.

\begin{theorem}[Sparse multinomial naive Bayes]\label{thm:sparse_mnb} \label{thm:cvx-rlx}
\; Let $\phi(k)$ be the optimal value of \eqref{eq:mnb0}. Then $\phi(k) \leq \psi(k)$, where $\psi(k)$ is the optimal value of the following one-dimensional convex optimization problem
\begin{equation}\label{eq:ub}\tag{USMNB}
    \psi(k) := C + \min_{\alpha \in [0,1]} \: s_k(h(\alpha) ),
    %\ones^\top(\fp+\fm) + \min_{\mu = (\mup,\mum)>0} \: \mup + \mum + \ones^\top v(\mu) + s_{k}(w(\mu)-v(\mu)) , 
\end{equation}
where $C$ is a constant, $s_{k}(\cdot)$ is the sum of the top $k$ entries of its vector argument, and for $\alpha \in (0,1)$, 
\begin{align*}
   &h(\alpha) = \tilde{C} - \fp \log \alpha - \fm  \log(1-\alpha). 
\end{align*}
where $\tilde{C} = \fp \circ \log \fp + \fm \circ \log \fm - (\fp+\fm) \circ \log (\fp+\fm)$. Furthermore, given an optimal dual variable $\alphao$ that solves \eqref{eq:ub}, we can reconstruct a primal feasible (sub-optimal) point $(\tp,\tm)$ for \eqref{eq:mnb0} as follows. For $\alpha^\ast$ optimal for \eqref{eq:ub}, let $\mathcal{I}$ be complement of the set of indices corresponding to the top $k$ entries of $h(\alphao)$; then set $B_\pm := \sum_{i \not\in\mathcal{I}} \fpm_i$, and
\begin{align}\label{eq:primalsol}
\tpoi &= \tmoi = \dfrac{\fp_i + \fm_i}{\ones^\top (\fp+\fm)}, \;\forall i \in \mathcal{I} \nonumber \\
    \tpmoi &= \dfrac{B_+ + B_-}{B_\pm} \dfrac{\fpm_i}{\ones^\top (\fp+\fm)}, \;\forall i \not\in \mathcal{I} 
\end{align}
\end{theorem}
\begin{proof}
We begin by deriving the expression for the upper bound $\psi(k)$.

\subsection{Duality bound}
We first derive the bound stated in the theorem.
Problem~\eqref{eq:mnb0} is written
\begin{align*}
    (\tpo,\tmo)  &= 
    \arg\max_{\tp, \tm \in [0,1]^m} \: \fpt \log \tp + \fmt \log \tm ~:~ 
    \begin{array}[t]{l} \ones^\top \tp = \ones^\top\tm = 1, \\ \|\tp - \tm\|_0 \leq k. 
    \end{array} \tag{SMNB} 
\end{align*}
By weak duality we have $\phi(k) \leq \psi(k)$ where
\begin{align*}
\psi(k) := \min_{\substack{\mup,\mum\\ \lambda \geq 0}} \: 
\max_{\tp, \tm \in [0,1]^m} & \fpt \log \tp + \fmt \log \tm + \mup (1 - \textbf{1}^\top \tp) + \mum (1 - \textbf{1}^\top \tm) \\
&+ \lambda (k - \|\tp-\tm\|_0).
\end{align*}
The inner maximization is separable across the components of $\tp,\tm$ since $\|\tp-\tm\|_0 = \sum_{i=1}^m \textbf{1}_{\{\tp_i\neq \tm_i\}}$. To solve it, we thus only need to consider one dimensional problems written
\begin{align}\label{eq:dual-1dmax}
    \max_{q,r\in[0,1]} \fp_i \log q + \fm_i \log r - \mup q - \mum r - \lambda \mathbbm{1}_{\{q \not = r\}},
\end{align}
where $\fp_i,\fm_i>0$ and $\mupm >0$ are given.  We can split the max into two cases; one case in which $q = r$ and another when $q \not = r$, then compare the objective values of both solutions and take the larger one. Hence \eqref{eq:dual-1dmax} becomes
\[
    \max \Big( \max_{u \in [0,1]} \: (\fp_i + \fm_i) \log u - (\mup + \mum) u, \max_{q,r\in[0,1]} \: \fp_i \log q + \fm_i \log r - \mup q - \mum r -\lambda \Big).
\]
Each of the individual maximizations can be solved in closed form, with optimal point
\BEQ\label{eq:opt-theta}
    u^\ast = \dfrac{(\fp_i + \fm_i)}{\mup+\mum}, \quad q^\ast = \dfrac{\fp_i}{\mup}, \quad r^\ast = \dfrac{\fm_i}{\mum}.
\EEQ
Note that none of $u^\ast,q^\ast,r^\ast$ can be equal to either 0 or 1, which implies $\mup, \mum > 0$. Hence \eqref{eq:dual-1dmax} reduces to
\BEQ\label{eq:dual-max}
    \max \Big( (\fp_i + \fm_i) \log \Big( \dfrac{(\fp_i+\fm_i)}{\mup + \mum} \Big) ,  \fp_i \log \Big(\dfrac{\fp_i}{\mup} \Big) + \fm_i \log \Big(\dfrac{\fm_i}{\mum} \Big) - \lambda \Big) - (\fp_i + \fm_i).
\EEQ
We obtain, with $S:=\ones^\top (\fp+\fm)$,
\begin{align}\label{eq:ub2}
    \psi(k)  &= -S + \min_{\substack{\mup,\mum>0 \\ \lambda \geq 0}} \: \mup+\mum + \lambda k + \sum_{i=1}^m \max (v_i(\mu), w_i(\mu) - \lambda) .
\end{align}
where, for given $\mu=(\mup,\mum)>0$,
\[
v(\mu) := (\fp+\fm) \circ \log  \Big( \dfrac{\fp+\fm}{\mup+\mum}\Big) ,  \;\;
w(\mu) := \fp \circ \log \Big( \dfrac{\fp}{\mup} \Big) + \fm \circ \log \Big( \dfrac{\fm}{\mum} \Big).
\]
Recall the variational form of $s_{k}(z)$. For a given vector $z \ge 0$, Lemma~\ref{lem:sk} shows
\[
    s_{k}(z) = \min_{\lambda \geq 0} \: \lambda k  + \sum_{i=1}^m \max (0, z_i - \lambda) .
\]
Problem \eqref{eq:ub2} can thus be written
\BEAS
    \psi(k) &=& -S+\min_{\substack{\mu>0 \\ \lambda \geq 0}} \: \mup+\mum + \lambda k + \ones^\top v(\mu) + \sum_{i=1}^m \max (0 , w_i(\mu)-v_i(\mu) - \lambda) \\
    &=&  -S + \min_{\mu>0} \: \mup+\mum + \ones^\top v(\mu) +  s_{k}(w(\mu) - v(\mu)) ,
\EEAS
where  the last equality follows from $w(\mu)\ \ge v(\mu)$, valid for any $\mu>0$. To prove this, observe that the negative entropy function $x \rightarrow x\log x$ is convex, implying that its perspective $P$ also is. The latter is the function with domain $\reals_+ \times \reals_{++}$, and values for $x \ge0$, $t>0$ given by $P(x,t) = x \log (x/t)$.  Since $P$ is homogeneous and convex, we have, for any pair $z_+,z_-$ in the domain of $P$: $P(z_++z_-) \le P(z_+)+P(z_-)$. Applying this to $z_\pm := (\fpm_i,\mup_i)$ for given $i \in [m]$ results in $w_i(\mu) \ge v_i(\mu)$, as claimed.

We further notice that the map $\mu \rightarrow  w(\mu) - v(\mu)$ is homogeneous, which motivates the change of variables $\mu_\pm = t\, p_\pm$, where $t=\mu_++\mu_->0$ and $p_\pm >0$, $p_++p_- = 1$. The problem reads 
\begin{align*}
\psi(k)  &= -S + (\fp+\fm)^\top \log(\fp+\fm) +\min_{\substack{t>0, \: p>0,\\p_++p_-=1}} \left\{ t - S \log t + s_{k}(H(p))\right\} \\&
= C + \min_{p>0, \: p_++p_-=1} \: s_k(H(p)),
\end{align*}
where $C :=(\fp+\fm)^\top \log(\fp+\fm) - S \log S$, because $t = S$ at the optimum, and 
\[
H(p) := v - \fp \circ \log p_+ - \fm \circ \log p_-,
\]
with
\[
v= \fp \circ \log \fp + \fm \circ \log \fm - (\fp+\fm)\circ \log(\fp + \fm).
\]
Solving for $\psi(k)$ thus reduces to a 1D bisection
\[
\psi(k) = C +  \min_{\alpha \in [0,1]} \: s_k(h(\alpha)),
\]
where
\[
h(\alpha) := H(\alpha,1-\alpha) = v - \fp \log \alpha - \fm  \log(1-\alpha).
\]
This establishes the first part of the theorem. Note that it is straightforward to check that with $k=n$, the bound is exact: $\phi(n) = \psi(n)$.

\subsection{Primalization}
Next we recover a primal feasible (sub-optimal) point $(\tpsub,\tmsub)$ from the dual bound obtained before. Assume that $\alphao$ is optimal for the dual problem \eqref{eq:ub}. We sort the vector $h(\alphao)$ and find the indices corresponding to the top $k$ entries. Denote the complement of this set of indices by $\mathcal{I}$. These indices are then the candidates for which $\tp_i = \tm_i$ for $i \in \mathcal{I}$ in the primal problem to eliminate the cardinality constraint. Hence we are left with solving
\begin{align}\label{eq:prml-dual-lb}
    (\tpsub,\tmsub)  &= 
    \arg\max_{\tp, \tm \in [0,1]^m} \: \fpt \log \tp + \fmt \log \tm \\
   &\text{s.t.} \ones^\top \tp = \ones^\top\tm = 1, \nonumber \\ 
   &\tp_i = \tm_i , \;\; i \in \mathcal{I} \nonumber
\end{align}
or, equivalently
\begin{align}\label{eq:prml-dual-chnge}
    \max_{\theta,\tp,\tm, s \in [0,1]} & \sum_{i \in \mathcal{I}} (\fp_i + \fm_i) \log \theta_i + \sum_{i \not\in \mathcal{I}} (\fp_i \log \tp_i + \fm_i \log \tm_i) \\
    \text{s.t.} &\; \ones^\top \tp = \ones^\top \tm = 1-s , \;\;  \textbf{1}^\top \theta = s . \nonumber 
\end{align}
For given $\kappa \in [0,1]$, and $f \in \reals_{++}^m$, we have
\[
\max_{u \::\: \ones^\top u = \kappa} \; f^\top \log(u)  = f^\top \log f - (\ones^\top f) \log (\ones^\top f) + (\ones^\top f) \log \kappa,
\]
with optimal point given by $u^\ast = (\kappa/(\ones^\top f))f$. Applying this to problem \eqref{eq:prml-dual-chnge}, we obtain that the optimal value of $s$ is given by
\[
s^\ast = \arg\max_{s \in (0,1)} \: \{ A \log s + B \log (1-s) \} = \frac{A}{A+B} ,
\]
where
\[
A := \sum_{i \in \mathcal{I}} (\fp_i+\fm_i) , \;\; B_\pm := \sum_{i \not\in \mathcal{I}} \fpm_i, \;\; B :=B_++B_- =  \ones^\top (\fp+\fm) - A.
\]
We obtain
\[
\tpsub_i = \tmsub_i =  \frac{s^\ast}{A}(\fp_i+\fm_i), \;\; i \in \mathcal{I}, \;\; 
\tpmsub_i = \frac{(1-s^\ast)}{B_\pm (A+B)} \fpm_i, \;\;  i \not\in \mathcal{I},
\]
which further reduces to the expression stated in the theorem.
\end{proof}

The key point here is that, while problem~\eqref{eq:mnb0} is nonconvex and potentially hard, the dual problem is a one-dimensional convex optimization problem which can be solved very efficiently, using bisection. The number of iterations to localize an optimal $\alpha^*$ with absolute accuracy $\epsilon$ grows slowly, as $O(\log(1/\epsilon))$; each step involves the evaluation of a sub-gradient of the objective function, which requires finding the $k$ largest elements in a $m$-vector, and costs $O(m\log k)$. As before in the Bernoulli case, the complexity of the sparse variant in the multinomial case is $O(mn + m\log k)$, versus $O(mn)$ for the classical naive Bayes. \textcolor{changescolor}{We summarize our method in Algorithm \ref{alg:smnb}.}

\textcolor{changescolor}{
\begin{algorithm}[h]
\caption{Sparse Multinomial Naive Bayes}\label{alg:smnb}
\begin{algorithmic}[1]
\State \textbf{Input: } $f_+, f_- \in \mathbb{R}^m$  
\State \;\; Set $v = f^+ \circ \log f^+ + f^- \circ \log f^- - (f^+ + f^-) \circ \log (f^+ + f^-)$
\State \;\; Solve $\alpha^\ast = \arg\min_{\alpha \in [0,1]} s_k (h(\alpha))$ \; where $h(\alpha) = v - \log \alpha f^+  - \log(1-\alpha) f^has -$
\State \;\; Set $\mathcal{I}$ be the indices of the smallest $m-k$ entries of $h(\alpha^\ast)$.
\State \;\; Compute $A = \sum_{i \in \mathcal{I}} (f_i^+ + f_i^-), \;\; B_\pm = \sum_{i \not \in \mathcal{I}} f_i^\pm, \;\; B = B_+ + B_-$ and $s = \dfrac{A}{A+B}$
\State \textbf{Output:} $\theta^\pm \in \mathbb{R}^m$ where $ \; \theta_i^\pm = \dfrac{s}{A} (f_i^+ + f_i^-), \;\; i \in \mathcal{I}$ and $\theta^\pm_i = \dfrac{1 - s}{B_\pm (A + B)} f_i^\pm, \; i \not \in \mathcal{I}$
\end{algorithmic}
\end{algorithm}
}

\subsection{Quality estimate} The quality of the bound in the multinomial case can be analysed using bounds on the duality gap based on the Shapley-Folkman theorem. Before we arrive at our result, we state and prove some technical lemmas.\\

Our quality estimate follows from results by \cite{Aubi76} (see also \cite{Ekel99,dAsp17} for a more recent discussion) which are briefly summarized below for the sake of completeness. Given functions $f_i$, a vector $b \in \reals^m$, and vector-valued functions $g_i$, $i\in[n]$ that take values in $\reals^m$, we consider the following problem: 
\BEQ\label{eq:p-ncvx-pb-const}\tag{P}
\mathrm{h}_{P}(u) := \min_x \: \sum_{i=1}^{n} f_i(x_i) ~:~ \sum_{i = 1}^n g_i(x_i) \leq b + u
\EEQ
in the variables $x_i\in\reals^{d_i}$, with perturbation parameter $u\in\reals^m$. We first recall some basic results about conjugate functions and convex envelopes. 

\paragraph{Biconjugate and convex envelope}
Given a function $f$, not identically $+\infty$, minorized by an affine function, we write
\[
f^*(y)\triangleq \inf_{x\in\dom f} \{y^{\top}x - f(x)\}
\]
the conjugate of $f$, and $f^{**}(y)$ its biconjugate. The biconjugate of $f$ (aka the convex envelope of $f$) is the pointwise supremum of all affine functions majorized by $f$ (see e.g. \cite[Th.\,12.1]{Rock70} or \cite[Th.\,X.1.3.5]{Hiri96}), a corollary then shows that $\epi(f^{**})=\overline{\Co(\epi(f))}$. For simplicity, we write $S^{**}=\overline{\Co(S)}$ for any set $S$ in what follows. We will make the following technical assumptions on the functions $f_i$ and $g_i$ in our problem.
\begin{assumption}\label{as:fi}
\; The functions $f_i: \reals^{d_i} \rightarrow \reals$ are proper, 1-coercive, lower semicontinuous and there exists an affine function minorizing them.
\end{assumption}
Note that coercivity trivially holds if $\dom(f_i)$ is compact (since $f$ can be set to $+\infty$ outside w.l.o.g.). When Assumption~\ref{as:fi} holds, $\epi(f^{**})$, $f_i^{**}$ and hence $\sum_{i=1}^{n} f_i^{**}(x_i)$ are closed \cite[Lem.\,X.1.5.3]{Hiri96}. Also, as in e.g. \cite{Ekel99}, we define the lack of convexity of a function as follows.

\begin{definition}\label{def:rho}
\; Let $f: \reals^{d} \rightarrow \reals$, we let 
\BEQ\label{eq:rho}
\rho(f)\triangleq \sup_{x\in \dom(f)} \{f(x) - f^{**}(x)\}
\EEQ
\end{definition}

Many other quantities measure lack of convexity (see e.g. \cite{Aubi76,Bert14} for further examples). In particular, the nonconvexity measure $\rho(f)$ can be rewritten as
\BEQ\label{def:alt-lack-cvx}
\rho(f)=\sup_{\substack{x_i\in \dom(f)\\ \mu\in\reals^{d+1}}}~\left\{ f\left(\sum_{i=1}^{d+1}\mu_i x_i\right) -  \sum_{i=1}^{d+1}\mu_i f(x_i): \ones^\top \mu=1,\mu \geq 0\right\}
\EEQ
when $f$ satisfies Assumption~\ref{as:fi} (see \cite[Th.\,X.1.5.4]{Hiri96}).

\paragraph{Bounds on the duality gap and the Shapley-Folkman Theorem}
Let $\mathrm{h}_{P}(u)^{**}$ be the biconjugate of $\mathrm{h}_{P}(u)$ defined in~\eqref{eq:p-ncvx-pb-const}, then $\mathrm{h}_{P}(0)^{**}$ is the optimal value of the dual to~\eqref{eq:p-ncvx-pb-const} \cite[Lem.\,2.3]{Ekel99}, and \cite[Th.\,I.3]{Ekel99} shows the following result.

\begin{theorem}\label{th:sf}
\; Suppose the functions $f_i,g_{ji}$ in problem~\eqref{eq:p-ncvx-pb-const} satisfy Assumption~\ref{as:fi} for $i=1,\ldots,n$, $j=1,\ldots,m$ \textcolor{changescolor}{where $g_{ji}$ denotes the $j$th entry of $g_i$ where $j \in [m]$}. Let
\BEQ\label{eq:sf-pbar}
\bar p_j = (m+1) \max_i \rho(g_{ji}), \quad \mbox{for $j=1,\ldots,m$}
\EEQ
then 
\BEQ\label{eq:sf-bnd}
\mathrm{h}_{P}(\bar p) \leq \mathrm{h}_{P}(0)^{**} + (m+1)\max_i \rho(f_i).
\EEQ
where $\rho(\cdot)$ is defined in Def.~\ref{def:rho}.
\end{theorem}

We are now ready to prove Theorem~\ref{thm:sparse_mnb_quality}, whose proof follows from Theorem~\ref{th:sf} above.

\begin{theorem}[Quality of Sparse Multinomial Naive Bayes Relaxation]\label{thm:sparse_mnb_quality}
\; Let $\phi(k)$ be the optimal value of \eqref{eq:mnb0} and $\psi(k)$ that of the convex relaxation in~\eqref{eq:ub}. We have for $k \ge 4$,
\begin{align}\label{eq:gap-bnd}
    \psi(k-4) \leq \phi(k) \leq \psi(k) \leq \phi(k + 4)
\end{align}
for $k\geq 4$.
\end{theorem}
\begin{proof}
Problem~\eqref{eq:mnb0} is {\em separable} and can be written in perturbation form as in the result by \cite[Th.\,I.3]{Ekel99} recalled in Theorem~\ref{th:sf}, to get
\BEQ\label{eq:pert-p}
\BA{rll}
    \mathrm{h}_{P}(u) = & \min_{q,r} & -\fpt \log q -\fmt \log r\\
    & \text{subject to} & \textbf{1}^\top  q = 1 + u_1,\\
    & & \textbf{1}^\top  r = 1 + u_2,\\
    & & \sum_{i=1}^m \ones_{q_i \neq r_i} \leq k + u_3
\EA\EEQ
in the variables $q,r\in [0,1]^m$, where $u\in\reals^3$ is a perturbation vector. By construction, we have $\phi(k)=-\mathrm{h}_{P}(0)$ and $\phi(k+l)=-\mathrm{h}_{P}((0,0,l))$. Note that the functions $\ones_{q_i \neq r_i}$ are lower semicontinuous and, because the domain of problem~\eqref{eq:mnb0} is compact, the functions
\[
\fp_i \log q_i + q_i + \fm_i \log r_i + r_i + \ones_{q_i \neq r_i}
\]
are 1-coercive for $i=1,\ldots,m$ on the domain and satisfy Assumption~\ref{as:fi} above. 

Now, because $q,r\geq 0$ with $\ones^\top  q=\ones^\top  r=1$, we have $q-r\in[-1,1]^m$ and the convex envelope of $\ones_{q_i \neq r_i}$ on $q,r\in [0,1]^m$ is $|q_i-r_i|$, hence the {\em lack of convexity}~\eqref{def:alt-lack-cvx} of $\ones_{q_i \neq r_i}$ on $[0,1]^2$ is bounded by one, because
\[
\rho(\ones_{x \neq y}) := \sup_{x,y\in[0,1]} \{\ones_{y \neq x} - |x-y|\} = 1 
\]
which means that $\max_{i=1,\ldots,n} \rho(g_{3i})=1$ in the statement of Theorem~\ref{th:sf}. The fact that the first two constraints in problem~\eqref{eq:pert-p} are convex means that $\max_{i=1,\ldots,n} \rho(g_{ji})=0$ for $j=1,2$, and the perturbation vector in~\eqref{eq:sf-pbar} is given by $\bar p=(0,0,4)$, because there are three constraints in problem~\eqref{eq:pert-p} so $m=3$ in~\eqref{eq:sf-pbar}, hence
\[
\mathrm{h}_{P}(\bar p)= \mathrm{h}_{P}((0,0,4))=-\phi(k+4).
\]
The objective function being convex separable, we have $\max_{i=1,\ldots,n} \rho(f_i)=0$. Theorem~\ref{th:sf} then states that
\[
\mathrm{h}_{P}(\bar p)= \mathrm{h}_{P}((0,0,4))= -\phi(k+4) \leq \mathrm{h}_{P}(0)^{**} + 0 = -\psi(k)
\]
because $-\mathrm{h}_{P}(0)^{**}$ is the optimal value of the dual to $\phi(k)$ which is here $\psi(k)$ defined in Theorem~\ref{thm:sparse_mnb}. The other bound in~\eqref{eq:gap-bnd}, namely $\phi(k) \leq \psi(k)$,
follows directly from weak duality.
\end{proof}

The bound in Theorem \ref{thm:sparse_mnb_quality} implies in particular
\[
\psi(k-4) \leq \phi(k) \leq \psi(k-4) + \Delta(k), \mbox{ for $k\geq4$,}
\]
where $\Delta(k):=\psi(k)-\psi(k-4)$. This means that if $\psi(k)$ does not vary too fast with $k$, so that $\Delta(k)$ is small, then the duality gap in problem~\eqref{eq:mnb0} is itself small, bounded by $\Delta(k)$; then solving the convex problem \eqref{eq:ub} will yield a good approximate solution to~\eqref{eq:mnb0}. This means that when the marginal contribution of additional features, i.e. $\Delta(k)/\psi(k)$ becomes small, our bound becomes increasingly tight. The ``elbow heuristic'' is often used to infer the number of relevant features $k^*$, with $\psi(k)$ increasing fast when $k<k^*$ and much more slowly when $k\geq k^*$. In this scenario, our bound becomes tight for $k\geq k^*$ .\\

\subsection{Primalization} While Theorem \ref{thm:sparse_mnb_quality} tells us that the objective costs of the primal and dual problems are close, it makes no statement about the solutions actually learned. Furthermore, since we extract a primal feasible (sub-optimal) solution via the dual, we are interested in bounding the quality of this point with regard to the true optimum point. We work towards another quality estimate, but first state and prove some technical lemmas.\\

We first derive the second dual of problem~\eqref{eq:p-ncvx-pb-const}, i.e. the dual of problem~\eqref{eq:ub}, which will be used to extract good primal solutions.
\begin{proposition}\label{prop:bidual}
\; A dual of problem~\eqref{eq:ub} is written
\BEQ\label{eq:p-cvx-pb-const}\tag{D}
\BA{ll}
\mbox{max.} & z^\top (g\circ \log(g))+ x^\top (\fp\circ\log(\fp)+\fm\circ\log(\fm)) + (x^\top g)\log (x^\top g) - (x^\top g)\\ 
\\
& - (\ones^\top g) \log(\ones^\top g) - (x^\top \fp) \log(x^\top \fp) - (x^\top \fm) \log(x^\top \fm)  \\
&\\
\mbox{s.t.} & x+z=\ones, \;\; \ones^\top x\leq k, \;\; x\geq 0, \;\; z\geq 0
\EA\EEQ
in the variables $x,z\in\reals^n$. Furthermore, strong duality holds between the dual~\eqref{eq:ub} and its dual~\eqref{eq:p-cvx-pb-const}.
\end{proposition}
\begin{proof}
The dual optimum value $\psi(k)$ in~\eqref{eq:ub} can be written as in~\eqref{eq:ub2}, 
\[
\psi(k)  = -S + \min_{\substack{\mup,\mum>0 \\ \lambda \geq 0}} \: \mup+\mum + \lambda k + \sum_{i=1}^m \max (v_i(\mu), w_i(\mu) - \lambda).
\]
with $S:=\ones^\top (\fp+\fm)$, and
\[
v(\mu) := (\fp+\fm) \circ \log  \Big( \dfrac{\fp+\fm}{\mup+\mum}\Big) ,  \;\;
w(\mu) := \fp \circ \log \Big( \dfrac{\fp}{\mup} \Big) + \fm \circ \log \Big( \dfrac{\fm}{\mum} \Big).
\]
for given $\mu=(\mup,\mum)>0$. This can be rewritten 
\[
\min_{\substack{\mup,\mum>0 \\ \lambda \geq 0}}~\max_{\substack{x+z=1\\x,z\geq 0}}~
\mup + \mum -S + \lambda (k-\ones^\top x) + z^\top v(\mu) +  x^\top w(\mu)\]
using additional variables $x,z\in\reals^n$, or again
\BEQ\label{eq:D-minmax}
\min_{\substack{\mup,\mum>0 \\ \lambda \geq 0}}~\max_{\substack{x+z=1\\x,z\geq 0}}~
\BA[t]{l}
 \lambda (k-\ones^\top x) - (x+z)^\top g - (z^\top g) \log(\mup+\mum) +z^\top (g\circ\log(g))\\
 - (x^\top \fp)\log(\mup) - (x^\top \fm) \log(\mum) \\
  + x^\top (\fp\circ\log(\fp)+\fm\circ\log(\fm))+ \mup + \mum
\EA\EEQ
calling $g=\fp+\fm$. Strong duality holds in this min max problem so we can switch the min and the max. 
Writing $\mu_\pm = t\, p_\pm$, where $t=\mu_+ +\mu_-$ and $p_\pm >0$, $p^+ +p^- = 1$ the Lagrangian becomes
\BEAS
L(p_+,p_-,t,\lambda,x,z,\alpha) &=&\ones^\top \nu - z^\top \nu - x^\top \nu + \lambda k - \lambda \ones^\top x - \ones^\top g - (z^\top g) \log(t)\\
&& - (x^\top \fp)\log(t\,p_+) - (x^\top \fm) \log(t\,p_-) + t\\
&& +z^\top (g\circ\log(g)) + x^\top (\fp\circ\log(\fp)+\fm\circ\log(\fm))\\
&& + \alpha (p_+ +p_- -1),
\EEAS
where $\alpha$ is the dual variable associated with the constraint $p_++p_-=1$. The dual of problem~\eqref{eq:ub} is then written
\BEAS
\sup_{\{x\geq 0,z\geq 0,\alpha\}} ~\inf_{\substack{p_+\geq 0,p_-\geq 0,\\t\geq 0,\lambda\geq 0}} L(p_+,p_-,t,\mum,\lambda,x,z,\alpha) 
\EEAS
The inner infimum will be $-\infty$ unless $\ones^\top x \leq k$, so the dual becomes
\[
\sup_{\substack{x+z=\ones, \ones^\top x\leq k,\\x\geq 0,z\geq 0,\alpha}} ~\inf_{\substack{p_+\geq 0,p_-\geq 0,\\t\geq 0}} 
\BA[t]{l}
z^\top (g\circ \log(g))+ x^\top (\fp\circ\log(\fp)+\fm\circ\log(\fm)) \\
- (x^\top \fp) (\log t + \log (p_+)) - (x^\top \fm) (\log t + \log (p_-))\\
+ t - \ones^\top  g -(z^\top g)\log(t) + \alpha(p_+ + p_- -1) 
\EA
\]
and the first order optimality conditions in $t,p_+, p_-$ yield
\BEA\label{eq:dual-opt}
t & = & \ones^\top g\\
p_+ & = & (x^\top \fp)/\alpha\nonumber\\
p_- & = & (x^\top \fm)/\alpha\nonumber
\EEA
which means the above problem reduces to
\[
\sup_{\substack{x+z=\ones, \ones^\top x\leq k,\\x\geq 0,z\geq 0,\alpha}}  
\BA[t]{l}
z^\top (g\circ \log(g))+ x^\top (\fp\circ\log(\fp)+\fm\circ\log(\fm)) \\
- (\ones^\top g) \log(\ones^\top g) - (x^\top \fp) \log(x^\top \fp) - (x^\top \fm) \log(x^\top \fm)  \\
+ (x^\top g)\log \alpha -\alpha
\EA\]
and setting in $\alpha=x^\top g$ leads to the dual in~\eqref{eq:p-cvx-pb-const}.
\end{proof}

We now use this last result to better characterize scenarios where the bound produced by problem~\eqref{eq:ub} is tight and recovers an optimal solution to problem~\eqref{eq:mnb0}. 

\begin{proposition}\label{prop:nonbinary}
\; Given $k>0$, let $\phi(k)$ be the optimal value of \eqref{eq:mnb0}. Given an optimal solution $(x,z)$ of problem~\eqref{eq:p-cvx-pb-const}, let $J=\{i:x_i \notin \{0,1\}\}$ be the set of indices where $x_i,z_i$ are not binary in $\{0,1\}$. There is a feasible point $\bar \theta,\bar \theta^+,\bar \theta^-$ of problem~\eqref{eq:mnb0} for $\bar k = k + |J|$, with objective value OPT such that
\[
\phi(k) \leq OPT \leq \phi(k + |J|).
\]
\end{proposition}
\begin{proof}
Using the fact that 
\[
\max_{x}~ a \log(x) - bx = a \log \left(\frac{a}{b}\right) -a
\]
the max min problem in \eqref{eq:D-minmax} can be rewritten as
\BEQ\label{eq:max-min-max}
\max_{\substack{x+z=1\\x,z\geq 0}}~\min_{\substack{\mup,\mum>0 \\ \lambda \geq 0}}~\max_{\theta,\theta^+,\theta^-}~
\BA[t]{l}
 \lambda (k-\ones^\top x) + z^\top (g\circ \log\theta)\\
 + x^\top( \fp \circ \log\theta^+) + x^\top (\fm \circ \log\theta^-)\\
  + \mup(1 - z^\top \theta - x^\top \theta^+) + \mum(1 - z^\top \theta - x^\top \theta^-)
\EA\EEQ
in the additional variables $\theta,\theta^+,\theta^-\in\reals^n$, with~\eqref{eq:opt-theta} showing that 
\[
\theta_i = \dfrac{(\fp_i + \fm_i)}{\mup+\mum}, \quad \theta_i^+ = \dfrac{\fp_i}{\mup}, \quad \theta_i^- = \dfrac{\fm_i}{\mum}.
\]
at the optimum. Strong duality holds in the inner min max, which means we can also rewrite problem~\eqref{eq:p-cvx-pb-const} as
\BEQ\label{eq:d-epi}
\max_{\substack{x+z=1\\x,z\geq 0}}~\max_{\substack{z^\top \theta + x^\top \theta^+\leq 1\\z^\top \theta + x^\top \theta^-\leq 1\\ x^\top \ones \leq k}}~ z^\top (g\circ \log\theta) + x^\top( \fp \circ \log\theta^+ + \fm \circ \log\theta^-)
\EEQ
or again, in epigraph form
\BEQ\label{eq:dual-epi}
\BA{ll}
\mbox{max.} & r\\
\mbox{s.t.} & 
\begin{pmatrix}
r\\
1\\
1\\
k
\end{pmatrix}
\in 
\begin{pmatrix}
0\\
\reals_+\\
\reals_+\\
\reals_+
\end{pmatrix}
+
\sum_{i=1}^n \left\{ z_i
\begin{pmatrix}
g_i \log\theta_i\\
\theta_i\\
\theta_i\\
0 
\end{pmatrix}
+ x_i
\begin{pmatrix}
\fp_i\log\theta_i^+ + \fm_i\log\theta_i^-\\
\theta_i^+\\
\theta_i^-\\
1
\end{pmatrix}
\right\}
\EA\EEQ
Suppose the optimal solutions $x^\star,z^\star$ of problem~\eqref{eq:p-cvx-pb-const} are binary in $\{0,1\}^n$ and let $\mathcal{I}=\{i:z_i=0\}$, then problem~\eqref{eq:d-epi} (hence problem~\eqref{eq:p-cvx-pb-const}) reads
\begin{align}
    (\tpsub,\tmsub)  = \arg\max_{\tp, \tm \in [0,1]^m}& 
     \: \fpt \log \tp + \fmt \log \tm \\
   \text{s.t.} &  \ones^\top \tp = \ones^\top\tm = 1, \nonumber \\ 
   &\tp_i = \tm_i , \;\; i \in \mathcal{I}. \nonumber
\end{align}
which is exactly \eqref{eq:prml-dual-lb}. This means that the optimal values of problem~\eqref{eq:prml-dual-lb} and~\eqref{eq:p-cvx-pb-const} are equal, so that the relaxation is tight and $\theta_i^+=\theta_i^-$ for $i\in\mathcal{I}$.
Suppose now that some coefficients $x_i$ are not binary. Let us call $J$ the set $J=\{i:x_i \notin \{0,1\}\}$. As in \cite[Th.~I.3]{Ekel99}, we define new solutions $\bar \theta,\bar \theta^+,\bar \theta^-$ and $\bar x, \bar z$ as follows,
\[
\left\{\BA{ll}
\bar \theta_i=\theta_i,~\bar \theta^+_i=\theta^+_i,~\bar \theta^-_i=\theta^-_i \mbox{ and } \bar z_i = z_i,~\bar x_i = x_i & \mbox{if $i \notin J$}\\
\bar \theta_i=0,~\bar \theta^+_i = z_i\theta + x_i\theta^+_i ,~\bar \theta^-_i = z_i\theta + x_i\theta^-_i \mbox{ and } \bar z_i = 0,~\bar x_i = 1 & \mbox{if $i \in J$}
\EA\right.\]
By construction, the points $\bar \theta,\bar \theta^+,\bar \theta^-$ and $\bar z, \bar x$ satisfy the constraints 
$\bar z^\top \bar \theta + \bar x^\top \bar \theta^+\leq 1$, $\bar z^\top \bar \theta + \bar x^\top \bar \theta^-\leq 1$ and $\bar x^\top \ones \leq k$. We also have $\bar x^\top \leq k + |J|$ and 
\BEAS
&& z^\top ((\fp+\fm)\circ \log\theta) + x^\top( \fp \circ \log\theta^+ + \fm \circ \log\theta^-)\\
&\leq & \bar z^\top ((\fp+\fm)\circ \log\bar\theta) + \bar x^\top( \fp \circ \log\bar\theta^+ + \fm \circ \log\bar\theta^-)
\EEAS
by concavity of the objective, hence the last inequality.
\end{proof}

We will now use the Shapley-Folkman theorem to bound the number of nonbinary coefficients in Proposition~\ref{prop:bidual} and construct a solution to~\eqref{eq:p-cvx-pb-const} satisfying the bound in Theorem~\ref{thm:sparse_mnb_quality}.

\begin{proposition}\label{prop:sf-feas}
\; There is a solution to problem~\eqref{eq:p-cvx-pb-const} with at most four nonbinary pairs $(x_i,z_i)$.
\end{proposition}
\begin{proof}
Suppose $(x^\star,z^\star,r^\star)$ and $(\theta,\theta^+_i,\theta^-_i)$ solve problem~\eqref{eq:p-cvx-pb-const} written as in~\eqref{eq:d-epi}, we get
\BEQ
\begin{pmatrix}
r^\star\\
1 - s_1\\
1 - s_2\\
k - s_3
\end{pmatrix}
=
\sum_{i=1}^n \left\{ z_i
\begin{pmatrix}
g_i \log\theta_i\\
\theta_i\\
\theta_i\\
0 
\end{pmatrix}
+ x_i
\begin{pmatrix}
\fp_i\log\theta_i^+ + \fm_i\log\theta_i^-\\
\theta_i^+\\
\theta_i^-\\
1
\end{pmatrix}
\right\}
\EEQ
where $s_1,s_2,s_3 \geq 0$. This means that the point $(r^\star, 1 - s_1,1 - s_1,k - s_3)$ belongs to a Minkowski sum of segments, with
\BEQ\label{eq:mink}
\begin{pmatrix}
r^\star\\
1 - s_1\\
1 - s_2\\
k - s_3
\end{pmatrix}
\in
\sum_{i=1}^n \Co \left( \left\{
\begin{pmatrix}
g_i \log\theta_i\\
\theta_i\\
\theta_i\\
0 
\end{pmatrix}
,
\begin{pmatrix}
\fp_i\log\theta_i^+ + \fm_i\log\theta_i^-\\
\theta_i^+\\
\theta_i^-\\
1
\end{pmatrix}
\right\}\right)
\EEQ
The Shapley-Folkman theorem~\cite{Star69} then shows that 
\BEAS
\begin{pmatrix}
r^\star\\
1 - s_1\\
1 - s_2\\
k - s_3
\end{pmatrix}
& \in &
\sum_{[1,n] \setminus \mathcal {S}}  \left\{
\begin{pmatrix}
g_i \log\theta_i\\
\theta_i\\
\theta_i\\
0 
\end{pmatrix}
,
\begin{pmatrix}
\fp_i\log\theta_i^+ + \fm_i\log\theta_i^-\\
\theta_i^+\\
\theta_i^-\\
1
\end{pmatrix}
\right\}\\
&& +
\sum_{\mathcal {S}} \Co \left( \left\{
\begin{pmatrix}
g_i \log\theta_i\\
\theta_i\\
\theta_i\\
0 
\end{pmatrix}
,
\begin{pmatrix}
\fp_i\log\theta_i^+ + \fm_i\log\theta_i^-\\
\theta_i^+\\
\theta_i^-\\
1
\end{pmatrix}
\right\}\right)
\EEAS
where $|\mathcal{S}|\leq 4$, which means that there exists a solution to~\eqref{eq:p-cvx-pb-const} with at most four nonbinary pairs $(x_i,z_i)$ with indices $i\in\mathcal{S}$. 
\end{proof}

In our case, since the Minkowski sum in~\eqref{eq:mink} is a polytope (as a Minkowski sum of segments), the Shapley-Folkman result reduces to a direct application of the fundamental theorem of linear programming, which allows us to reconstruct the solution of Proposition~\ref{prop:sf-feas} by solving a linear program. 

\begin{proposition}\label{prop:post}
\; Given $(x^\star,z^\star,r^\star)$ and $(\theta,\theta^+_i,\theta^-_i)$ solving problem~\eqref{eq:p-cvx-pb-const}, we can reconstruct a solution $(x,z)$ solving problem~\eqref{prop:bidual}, such that at most four pairs $(x_i,z_i)$ are nonbinary, by solving
\BEQ\label{eq:lp-postprocess}
\BA{ll}
\mbox{min.} & c^\top x\\
\mbox{s.t.} & \sum_{i=1}^n (1-x_i) g_i \log\theta_i + x_i (\fp_i\log\theta_i^+ + \fm_i\log\theta_i^-) = r^\star\\
& \sum_{i=1}^n (1-x_i) \theta_i + x_i \theta^+_i \leq 1\\
& \sum_{i=1}^n (1-x_i) \theta_i + x_i \theta^-_i \leq 1\\
& \sum_{i=1}^n x_i \leq k\\
& 0 \leq x \leq 1 
\EA\EEQ
which is a linear program in the variable $x\in\reals^n$ where $c\in\reals^n$ is e.g. a i.i.d. Gaussian vector. 
\end{proposition}
\begin{proof}
Given $(x^\star,z^\star,r^\star)$ and $(\theta,\theta^+_i,\theta^-_i)$ solving problem~\eqref{eq:p-cvx-pb-const}, we can reconstruct a solution $(x,z)$ solving problem~\eqref{prop:bidual}, by solving~\eqref{eq:lp-postprocess} which is a linear program in the variable $x\in\reals^n$ where $c\in\reals^n$ is e.g. a i.i.d. Gaussian vector. This program has $2n+4$ constraints, at least $n$ of which will be saturated at the optimum. In particular, at least $n-4$ constraints in $0 \leq x \leq 1$ will be saturated so at least $n-4$ coefficients $x_i$ will be binary at the optimum, \textcolor{changescolor}{and similarly} for the corresponding coefficients $z_i=1-x_i$.
\end{proof}

Proposition~\ref{prop:post} shows that solving the linear program in~\eqref{eq:lp-postprocess} as a postprocessing step will produce a solution to problem~\eqref{eq:p-cvx-pb-const} with at most $n-4$ nonbinary coefficient pairs $(x_i,z_i)$. Proposition~\ref{prop:nonbinary} then shows that this solution satisfies
\[
\phi(k) \leq OPT \leq \phi(k + 4).
\]
which is the bound in Theorem~\eqref{thm:sparse_mnb_quality}.

Finally, we show a technical lemma linking the dual solution $(x,z)$ in~\eqref{eq:p-cvx-pb-const} above and the support of the $k$ largest coefficients in the computation of $s_k(h(\alpha))$ in theorem~\ref{thm:sparse_mnb}.

\begin{lemma}\label{lem:sk}
\; Given $c\in\reals^n_+$, we have
\BEQ\label{eq:sk-min}
    s_{k}(c) = \min_{\lambda \geq 0} \: \lambda k  + \sum_{i=1}^n \max (0, c_i - \lambda)
\EEQ
and given $k$, $\lambda \in [c_{[k+1]},c_{[k]}]$ at the optimum, where $c_{[1]} \geq \ldots \geq c_{[n]}$. Its dual is written
\BEQ\label{eq:dual-mink}
\BA{ll}
\mbox{max.} & x^\top c\\
\mbox{s.t.} & \ones^\top x \leq k\\
            & x + z = 1 \\
            & 0 \leq z,x
\EA\EEQ
When all coefficients $c_i$ are distinct, the optimum solutions $x,z$ of the dual have at most one nonbinary coefficient each, i.e. $x_i,z_i \in (0,1)$ for a single $i \in [1,n]$. If in addition $c_{[k]}>0$, the solution to~\eqref{eq:dual-mink} is binary.
\end{lemma}
\begin{proof}
Problem \eqref{eq:sk-min} can be written
\[\BA{ll}
\mbox{min.} & \lambda k  + \ones^\top t\\
\mbox{s.t.} & c - \lambda \ones \leq t\\
            & 0 \leq t\\
\EA\]
and its Lagrangian is then
\[
L(\lambda,t,z,x) = \lambda k  + \ones^\top t + x^\top(c - \lambda \ones - t) + z^\top t.
\]
The dual to the minimization problem~\eqref{eq:sk-min} reads
\[\BA{ll}
\mbox{max.} & x^\top c\\
\mbox{s.t.} & \ones^\top x \leq k\\
            & x + z = 1 \\
            & 0 \leq z,x
\EA\]
in the variable $w\in\reals^n$, its optimum value is $s_{k}(z)$. By construction, given $\lambda \in [c_{[k+1]},c_{[k]}]$, only the $k$ largest terms in $\sum_{i=1}^m \max (0, c_i - \lambda)$ are nonzero, and they sum to $s_{k}(c)-k \lambda$. The KKT optimality conditions impose 
\[
x_i(c_i - \lambda - t_i)=0
\quad \mbox{and} \quad
z_it_i=0, \quad i=1,\ldots,n
\]
at the optimum. This, together with $x + z = 1$ and $t,x,z\geq 0$, means in particular that 
\BEQ\label{eq:slack}
\left\{\BA{ll}
x_i=0, z_i=1, &\mbox{if } c_i-\lambda < 0\\
x_i=0, z_i=1, \mbox{ or } x_i=1, z_i=0  & \mbox{if } c_i-\lambda > 0
\EA\right.
\EEQ
the result of the second line comes from the fact that if $c_i-\lambda > 0$ and $t_i=c_i - \lambda$ then $z_i=0$ hence $x_i=1$, if on the other hand $t_i\neq c_i - \lambda$, then $x_i=0$ hence $z_i=1$. When the coefficients $c_i$ are all distinct, $c_i-\lambda=0$ for at most a single index $i$ and~\eqref{eq:slack} yields the desired result. When $c_{[k]}>0$ and the $c_i$ are all distinct, then the only way to enforce zero gap, i.e.
\[
x^\top c =  s_k(c)
\]
is to set the corresponding coefficients of $x_i$ to one.
\end{proof}

% \begin{theorem}[Quality of Sparse Multinomial Naive Bayes Relaxation]\label{thm:sparse_mnb_quality} \label{thm:cvx-rlx-quality} 
% Let $\phi(k)$ be the optimal value of \eqref{eq:mnb0} and $\psi(k)$ that of the convex relaxation in~\eqref{eq:ub}, we have, for $k \ge 4$,
% \BEQ\label{eq:gap-bnd}
%     \psi(k-4) \leq \phi(k) \leq \psi(k) \leq \phi(k + 4) .
% \EEQ
% \end{theorem}
% \begin{proof}
% See Appendix~\ref{appendixC}.
% \end{proof}

\section{Experiments}\label{sec:expt}
In this section,  we compare our sparse multinomial model \eqref{eq:mnb0} against other feature selection methods (Experiments 1-5) we empirically show the quality of our relaxation on a synthethic dataset (Experiment 6). For the former experiments, we do not use deep learning methods since we want to compare the features selected rather than the end-to-end training accuracy. For this reason, we compare \eqref{eq:mnb0} against traditional $\ell_1$ methods, recursive feature elimination (RFE) methods, and other sparsity-inducing methods. 

\begin{table*}[ht]
\centering
{\small
%\captionsetup{width=1\textwidth}
%\begin{adjustwidth}{3cm}{}
    \begin{tabular}{lrrrrr}
\toprule
     \textsc{Feature Vectors} &     \textsc{Amazon} & \textsc{IMDB} & \textsc{Twitter} & \textsc{MPQA} & \textsc{SST2}\\
\midrule
\textsc{Count Vector} &    \textsc{31,666} & \textsc{103,124} & \textsc{273,779} & \textsc{6,208} & \textsc{16,599} \\
\textsc{tf-idf} &    \textsc{5000} & \textsc{5000} & \textsc{5000} & \textsc{5000} & \textsc{5000} \\
\textsc{tf-idf wrd bigram} &   \textsc{5000} & \textsc{5000} & \textsc{5000} & \textsc{5000} & \textsc{5000} \\
\textsc{tf-idf char bigram} &   \textsc{5000} & \textsc{5000} & \textsc{5000} & \textsc{4838} & \textsc{5000} \\ \midrule
\textsc{$n_\text{train}$} &   \textsc{8000} & \textsc{25,000} & \textsc{1,600,000} & \textsc{8484} & \textsc{76,961}\\ 
\textsc{$n_\text{test}$} &   \textsc{2000} & \textsc{25,000} & \textsc{498} & \textsc{2122} & \textsc{1821}\\ 
\bottomrule\\
\end{tabular}}
%\end{adjustwidth}
\caption{\textbf{Experiment 1 data}: Number of features for each type of feature vector for each data set. For tf-idf feature vectors, we fix the maximum number of features to 5000 for all data sets. The last two rows show the number of training and test samples.}
    \label{tab:dsinfo}
\end{table*}

\begin{figure*}[h!]
    \centering
    \includegraphics[height=.45\textheight]{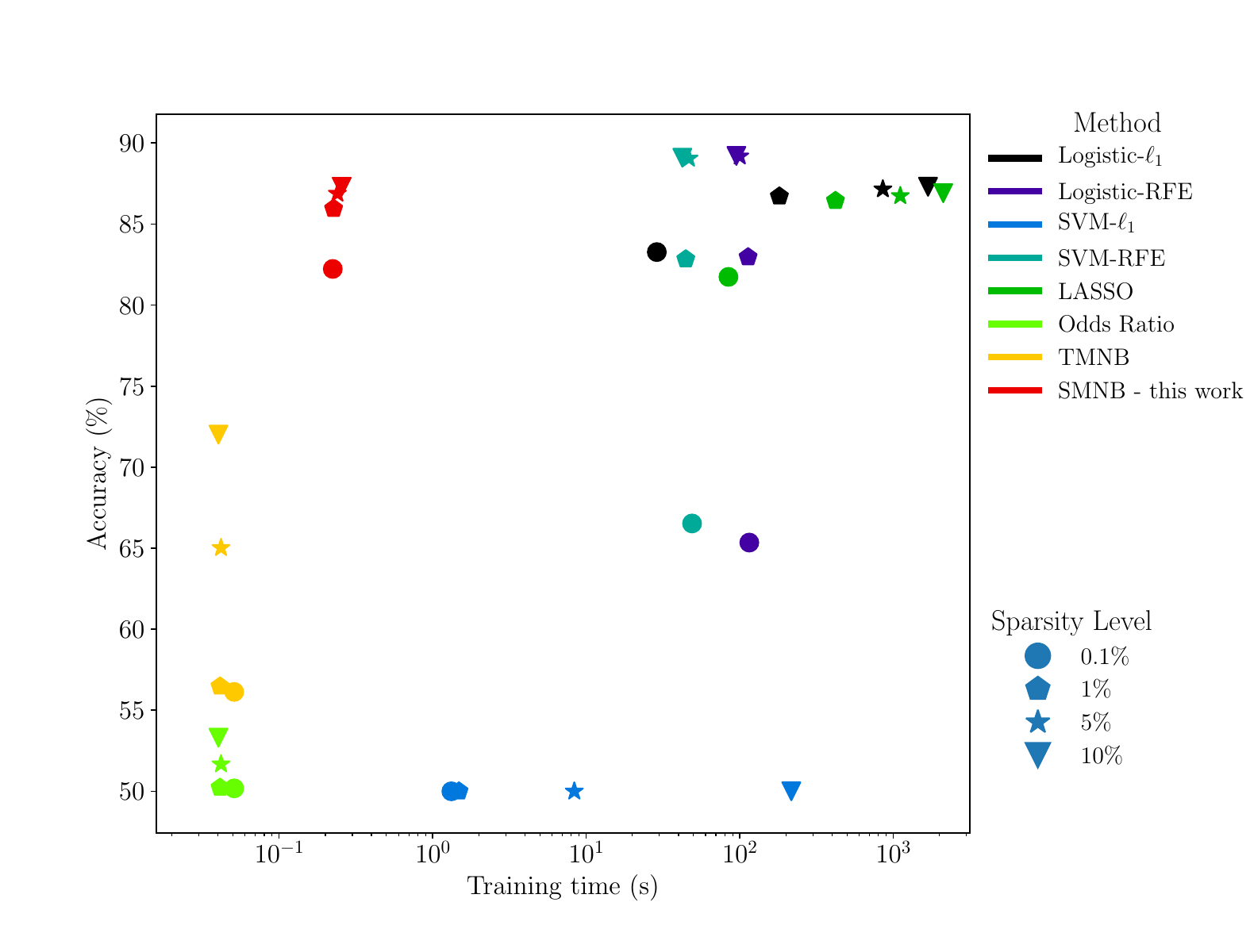} %for old plot, the figure name is expt2_imdb.pdf
    \caption{\textbf{Experiment 1:} Accuracy versus run time with the IMDB dataset/Count Vector with MNB in stage 2, showing performance on par with the best feature selection methods, at fraction of computing cost. Times \textit{do not} include the cost of grid search to reach the target cardinality for $\ell_1$-based methods. For more details on the experiment, see Appendix~\ref{appendixF}.}
    \label{fig:expt2}
\end{figure*}

% \begin{figure}[h!]
%     \centering
%     \subfloat[$m=30$]{{\includegraphics[width=0.46\linewidth]{figures/dualbound_30.png} }}
%     \qquad
%     \subfloat[$m=3000$]{{\includegraphics[width=0.46\linewidth]{figures/dualbound_3000.png} }}
%     \caption{\textbf{Experiment 1:} Duality gap bound versus sparsity level for $m = \{30,3000\}$}
%     \label{fig:expt1}
% \end{figure}

\begin{table*}[ht]
\centering
%\captionsetup{width=1\textwidth}
%\begin{adjustwidth}{3cm}{}
    {\small \begin{tabular}{lrrrrr}
\toprule
     \textsc{Feature Vectors} &     \textsc{Amazon} & \textsc{IMDB} & \textsc{Twitter} & \textsc{MPQA} & \textsc{SST2}\\
\midrule
\textsc{Count Vector} &    \textsc{31,666} & \textsc{103,124} & \textsc{273,779} & \textsc{6,208} & \textsc{16,599} \\
\textsc{tf-idf} &    \textsc{31,666} & \textsc{103,124} & \textsc{273,779} & \textsc{6,208} & \textsc{16,599} \\
\textsc{tf-idf wrd bigram} &   \textsc{870,536} & \textsc{8,950,169} & \textsc{12,082,555} & \textsc{27,603} & \textsc{227,012} \\
\textsc{tf-idf char bigram} &   \textsc{25,019} & \textsc{48,420} & \textsc{17,812} & \textsc{4838} & \textsc{7762} \\
\bottomrule\\
\end{tabular}}
%\end{adjustwidth}
\caption{\textbf{Experiment 2 data}: Number of features for each type of feature vector for each data set with no limit on the number of features for the tf-idf vectors. The train/test split is the same as in Table \ref{tab:dsinfo}.}
    \label{tab:dsinfo2}
\end{table*}

\begin{figure*}[ht]
  \centering
  \begin{minipage}[b]{0.485\textwidth}
    \includegraphics[width=\textwidth]{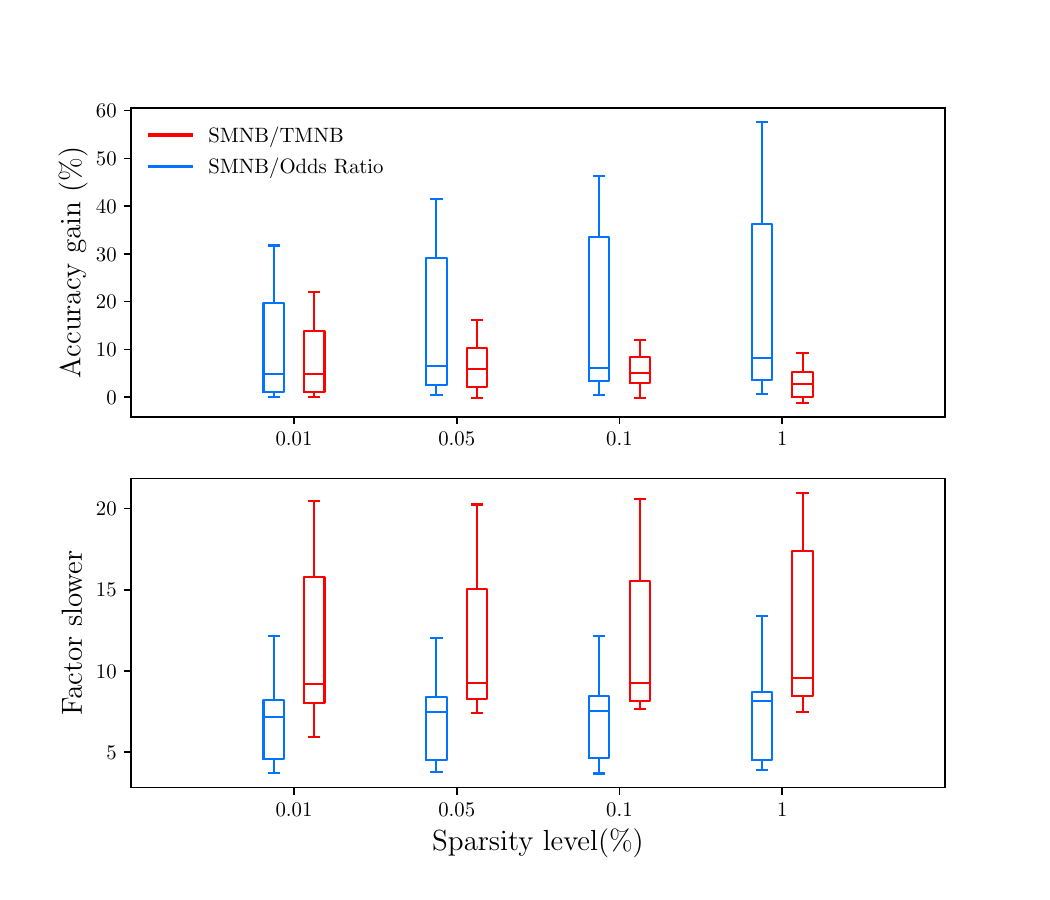}
  \end{minipage}
  \hfill
  \begin{minipage}[b]{0.485\textwidth}
    \includegraphics[width=\textwidth]{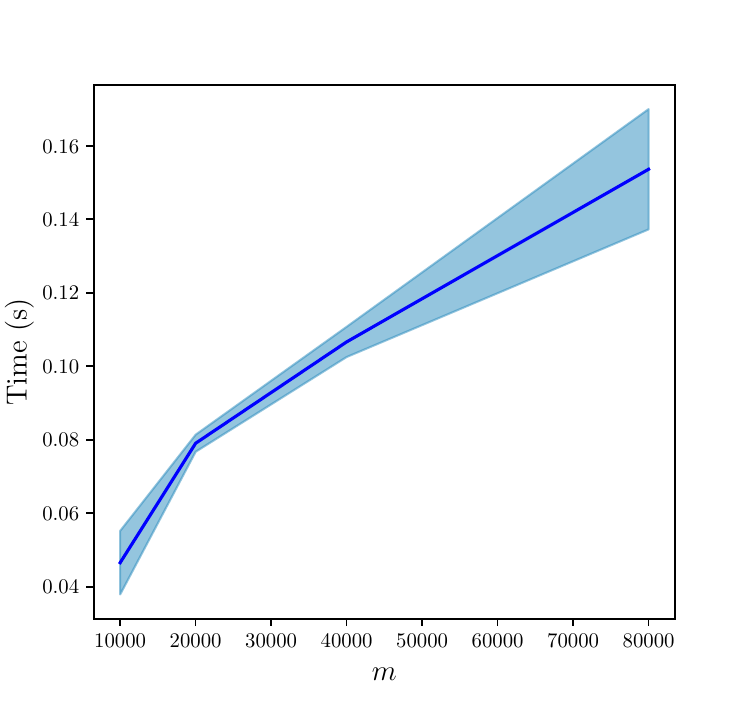}
  \end{minipage}
  \caption{\textbf{Experiment 2 (Left):} Accuracy gain for our method (top panel) and factor slower (bottom panel) over all data sets listed in Table \ref{tab:dsinfo2} with MNB in stage 2, showing substantial performance increase with a constant increase in computational cost. \textbf{Experiment 3 (Right):} Run time with IMDB dataset/tf-idf vector data set, with increasing $m,k$ with fixed ratio $k/m$, empirically showing (sub-) linear complexity.}
     \label{fig:expt34}
\end{figure*}

\subsection{Experiment 1: Feature Selection}
In the next three experiments, we compare  \eqref{eq:mnb0} with other feature selection methods for sentiment classification on five different text data sets. Some details on the data sets sizes are given in Table~\ref{tab:dsinfo}. More information on these data sets and how they were pre-processed are given in Appendix D of \cite{askari2020naive}. 

% \newlength{\oldintextsep}
% \setlength{\oldintextsep}{\intextsep}

For each data set and each type of feature vector, we perform the following two-stage procedure. In the first step,  we employ a feature selection method to attain a desired sparsity level of (0.1\%, 1\%, 5\%, 10\%); in the second step, we train a classifier based on the selected features. Specifically, we use $\ell_1$-regularized logistic regression, logistic regression with recursive feature elimination (RFE), $\ell_1$-regularized support vector machine (SVM), SVM with RFE, LASSO, thresholded Multinomial naive Bayes (TMNB), the Odds Ratio metric described by \cite{mladenic1999feature} and \eqref{eq:mnb0} in the first step. Then using the selected features, in the second step we train a logistic model, a MNB model, and a SVM. Thresholded multinomial naive Bayes (TMNB) means we train a multinomial naive Bayes model and then select the features corresponding to indices of the largest absolute value entries of the vector of classification coefficients $w_m$, as defined in \eqref{eq:m-rule}. For each desired sparsity level and each data set in the first step, we do a grid search over the optimal Laplace smoothing parameter for MNB for each type of feature vector. We use this same parameter in \eqref{eq:mnb0}. All models were implemented using Scikit-learn \cite{scikit}. Figure \ref{fig:expt2} shows that \eqref{eq:mnb0} is competitive with other feature selection methods, consistently maintaining a high test set accuracy, while only taking a fraction of the time to train; for a sparsity level of $5\%$, a logistic regression model with $\ell_1$ penalty takes more than $1000$ times longer to train.

\subsection{Experiment 2: Large-scale Feature Selection}
For this experiment, we consider the same data sets as before, but do not put any limit on the number of features for the tf-idf vectors. Due to the large size of the data sets, most of the feature selection methods in Experiment 1 are not feasible. We use the same two-stage procedure as before: 1) do feature selection using TMNB, the Odds Ratio method and our method \eqref{eq:ub}, and 2) train a MNB model using the features selected in stage 1. We tune the hyperparameters for MNB and \eqref{eq:ub} the same way as in Experiment 2. In this experiment, we focus on sparsity levels of $0.01\%, 0.05\%, 0.1\%, 1\%$.  Table \ref{tab:dsinfo2} summarizes the data used in Experiment 2 and in Table \ref{tab:time} we display the average training time for \eqref{eq:ub}.

\begin{table}[h!]
\centering
%\captionsetup{width=1\textwidth}
%\begin{adjustwidth}{3cm}{}
    {\small \begin{tabular}{lrrrrr}
\toprule
     \textsc{} &     \textsc{Amazon} & \textsc{IMDB} & \textsc{Twitter} & \textsc{MPQA} & \textsc{SST2}\\
\midrule
\textsc{$F_C$} &    \textsc{0.043} & \textsc{0.22} & \textsc{1.15} & \textsc{0.0082} & \textsc{0.037} \\
\textsc{$F_{t_1}$} &    \textsc{0.033} & \textsc{0.16} & \textsc{0.89} & \textsc{0.0080} & \textsc{0.027} \\
\textsc{$F_{t_2}$} &   \textsc{0.68} & \textsc{9.38} & \textsc{13.25} & \textsc{0.024} & \textsc{0.21} \\
\textsc{$F_{t_3}$} &   \textsc{0.076} & \textsc{0.47} & \textsc{4.07} & \textsc{0.0084} & \textsc{0.082} \\
\bottomrule\\
\end{tabular}}
%\end{adjustwidth}
\caption{\textbf{Experiment 2 run times}: Average run time (in seconds, with a standard CPU and a non-optimized implementation) over $4 \times 30 = 120$ values for different sparsity levels and 30 randomized train/test splits per sparsity level for each data set and each type of feature vector. On the largest data set (\textsc{Twitter}, $\sim 12$M features, $\sim 1.6$M data points), the computation takes less than 15 seconds. For the full distribution of run times, see Appendix \ref{appendixF}. $F_C, F_{t_1}, F_{t_2}, F_{t_3}$ refer to the count vector, tf-idf, tf-idf word bigram, and tf-idf character bigram feature vectors respectively.}
    \label{tab:time}
\end{table}

% \begin{figure}[h!]
%     \centering
%     \includegraphics[height=.45\textheight]{figures/expt3_gains.pdf}
%     \caption{\textbf{Experiment 3:} Accuracy gain for our method (top panel) and factor slower (bottom panel) over all data sets listed in Table \ref{tab:dsinfo2} with MNB in stage 2, showing substantial performance increase with a constant increase in computational cost.}
%     \label{fig:expt3}
% \end{figure}

Figure \ref{fig:expt34} shows that, even for large datasets with millions of features and data points, our method, implemented on a standard CPU with a non-optimized solver, takes at most a few seconds, while providing a significant improvement in performance. See Appendix D of \cite{askari2020naive} for the accuracy versus sparsity plot for each data set and each type of feature vector.

\subsection{Experiment 3: Complexity}
Using the IMDB dataset in Table \ref{tab:dsinfo}, we perform the following experiment: we fix a sparsity pattern $k/m = 0.05$ and then increase $k$ and $m$. Where we artificially set the number of tf-idf features to 5000 in Experiment 1, here we let the number of tf-idf features vary from $10,000$ to $80,000$. We then plot the the time it takes to train \eqref{eq:mnb0} at the fixed $5 \%$ sparsity level. Figure \ref{fig:expt34} shows that for a fixed sparsity level, the complexity of our method appears to be sub-linear. 

\subsection{Experiment 4: Topic Modeling}
In this experiment, we interpret the features learned by \eqref{eq:mnb0} for the IMDB dataset using count vectors. We set the sparsity level such that we recover only $9$ features. In this case, the features correspond to words that \eqref{eq:mnb0} uses to determine whether a movie review was positive or negative. Table \ref{tab:topics} contains the list of words and their associated weightings.
\begin{table}[h!]
\centering
%\captionsetup{width=1\textwidth}
%\begin{adjustwidth}{3cm}{}
    {\small \begin{tabular}{rrrrrrrrr}
\toprule
     \textsc{love} &     \textsc{best} & \textsc{waste} & \textsc{excellent} & \textsc{bad} & \textsc{great} & 
     \textsc{wonderful} & \textsc{worst} & \textsc{awful}\\
     \textsc{0.69} &     \textsc{0.72} & \textsc{-2.63} & \textsc{1.46} & \textsc{-1.36} & \textsc{0.88} & 
     \textsc{1.56} & \textsc{-2.29} & \textsc{-2.23}\\
\end{tabular}}
\caption{\textbf{Experiment 4}: List of top $9$ features/words learned by \eqref{eq:mnb0} and the associated weighting in the classifier.}
\label{tab:topics}
\end{table}
Table \ref{tab:topics} has a very natural interpretation. It assigns positive weightings to words that generally contain positive sentiment ("love", "best", "excellent", "great", "wonderful") and negative weightings to words with negative sentiment ("waste", "bad", "worst", "awful"). Additionally, the signs of all the weightings align with their respective sentiments (i.e. all words with positive weightings have positive sentiment and vice versa). This indicates that the model was able to match words directly with their sentiment instead of using some uninterpretable linear combination of words to perform sentiment classification. Using only 9 words (out of 103,124), \eqref{eq:mnb0} was able to obtain 70.2\% classification accuracy on the test set.

\subsection{Experiment 5: Genetic Data}
In this experiment, we use gene expression data as opposed to text data. Since biological data naturally lacks a ground truth baseline, we are instead interested in the tradeoff in the objective value of \eqref{eq:mnb0} with the sparsity level. Specifically we use the RNA-sequence data available at the UCI machine learning repository which contains 801 data points and 20531 features\cite{UCI}. We set up the binary classification problem by putting all the BRCA cells in one class and the rest in the other class. Figure \ref{fig:expt5} plots the change in objective value as we vary the sparsity level. We see that as the number of features increases beyond $\approx 10,000$, the objective value does not change. Furthermore we see steep improvement initially and then very gradual/slow improvement as more features are included. This indicates that we can achieve a similar objective value using just 10\% of the features as opposed to all of them.
\begin{figure*}[h!]
    \centering
    \includegraphics[scale=0.7]{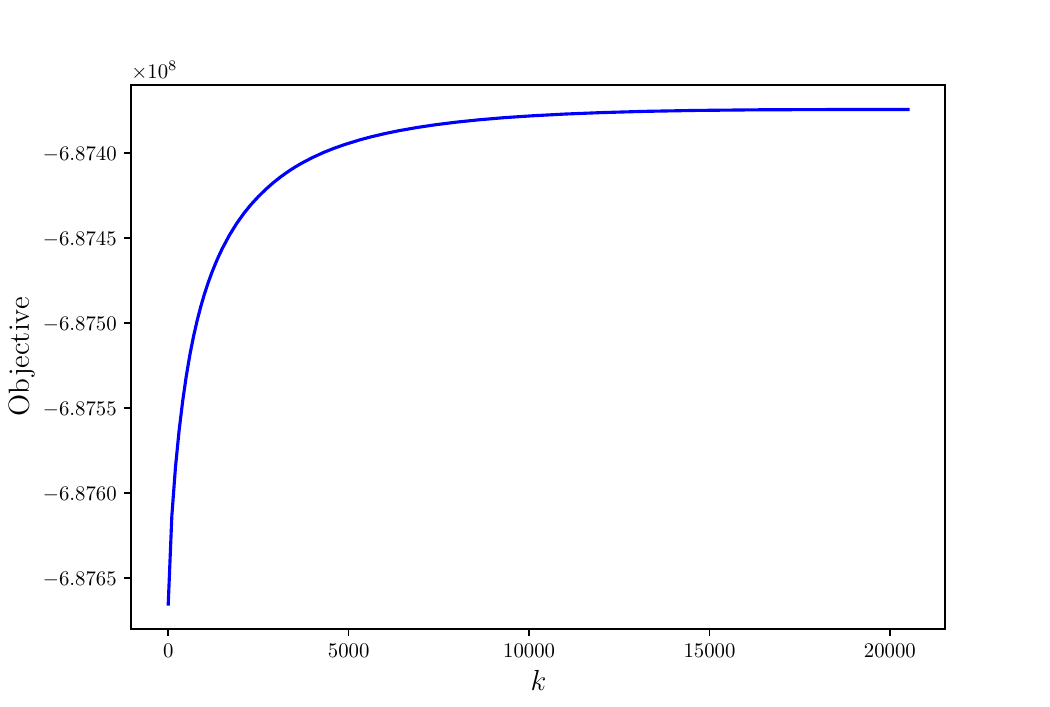} 
    \caption{\textbf{Experiment 5:} Tradeoff of objective value of \eqref{eq:mnb0} vs sparsity level $k$.}
    \label{fig:expt5}
\end{figure*}

\subsection{Experiment 6: Duality Gap}
% \begin{figure}[h!]
%     \centering
%     \includegraphics[height=0.19\textheight]{figures/dualbound.pdf}
%     \caption{\textbf{Experiment 1:} Duality gap bound versus sparsity level for $m \in \{30,3000\}$, showing that the duality gap quickly closes.}
%     \label{fig:expt1}
% \end{figure}

In this experiment, we generate random synthetic data with uniform independent entries: $\fpm \sim U[0,1]^m$, where $m \in \{30,3000\}$. We then normalize $\fpm$ and compute $\psi(k)$ and $\psi(k-4)$ for $4 \leq k \leq m$ and plot how this gap evolves as $k$ increases. For each value of $k$, we also plot the value of the reconstructed primal feasible point, as detailed in Theorem~\ref{thm:cvx-rlx}. The latter serves as a lower bound on the true value $\phi(k)$, which can be used to test \textit{a posteriori} if our bound is accurate.

\begin{figure}[h]
  \centering
  \begin{minipage}[b]{0.46\textwidth}
    \includegraphics[width=\textwidth]{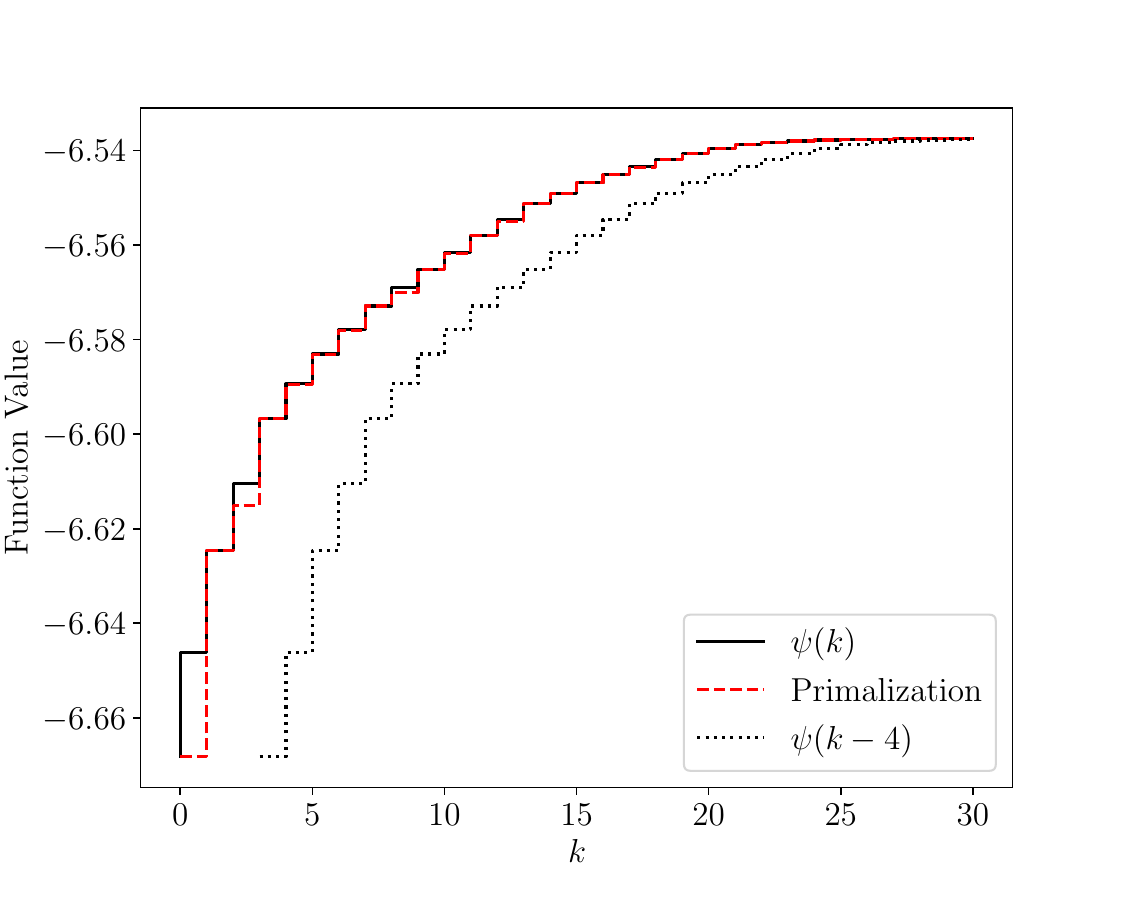}
  \end{minipage}
  \hfill
  \begin{minipage}[b]{0.46\textwidth}
    \includegraphics[width=\textwidth]{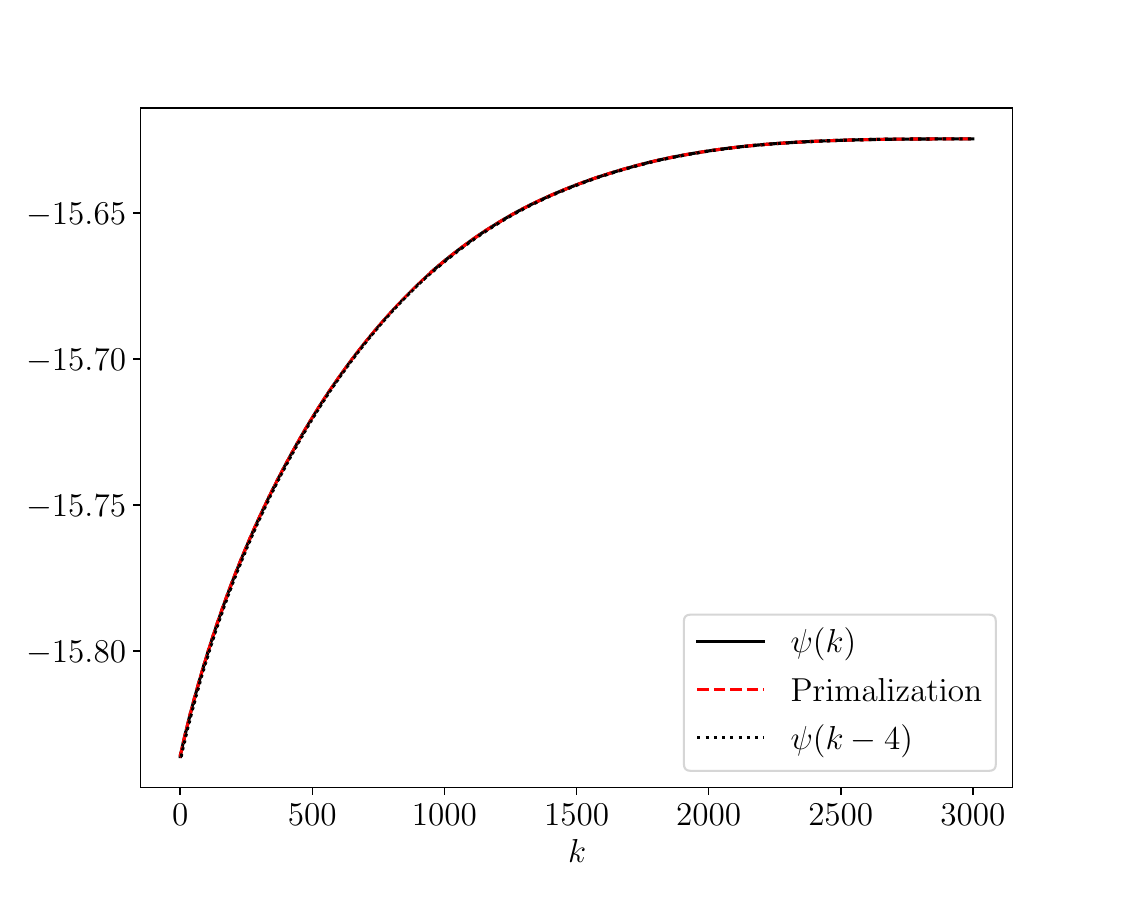}
  \end{minipage}
  \caption{\textbf{Experiment 6:} Duality gap bound versus sparsity level for $m = 30$ (top panel) and $m= 3000$ (bottom panel), showing that the duality gap quickly closes as $m$ or $k$ increase.}
     \label{fig:expt1}
\end{figure}

Figure \ref{fig:expt1} shows that, as the number of features $m$ or the sparsity parameter $k$ increases, the duality gap bound decreases. Figure \ref{fig:expt1} also shows that the \textit{a posteriori} gap is almost always zero, implying strong duality. In particular, as shown in Figure \ref{fig:expt1}(b), as the number of features increases, the gap between the bounds and the primal feasible point's value becomes negligible for all values of $k$. This indicates that we can solve the original, non-convex problem \eqref{eq:bnb0} by instead solving a 1-dimensional dual problem and constructing a primal feasible solution in closed form.

\section{Conclusion}
In this paper, we propose a sparse version of naive Bayes, leading to a combinatorial maximum likelihood problem that we show is more benign than it appears. In the case of binary data, we are able to solve the problem exactly, while in the multinomial case, we provide explicit bounds on the duality gap and show it decreases as as the marginal contribution of additional features decreases. Furthermore, we show empirically on synthetic data that this bound is quite loose and that our scheme appears to be tight (ie. strong duality holds). We test our method on different text data sets with other popular feature selection methods. On all the data sets, we are able to maintain the same performance on the test set while only taking a fraction of the time to train (in some cases our method is 1000x faster than other methods with specialized solvers).

\section*{Acknowledgements}
AdA is at the d\'epartement d'informatique de l'ENS, \'Ecole normale sup\'erieure, UMR CNRS 8548, PSL Research University, 75005 Paris, France, and INRIA Sierra project-team. AdA would like to acknowledge support from the {\em ML and Optimisation} joint research initiative with the {\em fonds AXA pour la recherche} and Kamet Ventures, a Google focused award, as well as funding by the French government under management of Agence Nationale de la Recherche as part of the "Investissements d'avenir" program, reference ANR-19-P3IA-0001 (PRAIRIE 3IA Institute). LEG would like to acknowledge support from Berkeley Artificial Intelligence Research (BAIR) and Tsinghua-Berkeley-Shenzhen Institute (TBSI).

%\newpage

\bibliographystyle{siamplain}
\bibliography{main}

\newpage

%\begin{APPENDIX}{Experimental Results}
\section{Experimental Results}
\section{Details on Datasets}\label{appendixF}
This section details the data sets used in our experiments.

\paragraph{Downloading data sets}
\begin{enumerate}
    \item \underline{AMZN} The complete Amazon reviews data set was collected from \href{https://drive.google.com/drive/folders/0Bz8a_Dbh9Qhbfll6bVpmNUtUcFdjYmF2SEpmZUZUcVNiMUw1TWN6RDV3a0JHT3kxLVhVR2M}{here}; only a subset of this data was used which can be found \href{https://gist.github.com/kunalj101/ad1d9c58d338e20d09ff26bcc06c4235}{here}. This data set was randomly split into 80/20 train/test.
    \item \underline{IMDB} The large movie review (or IMDB) data set was collected from \href{http://ai.stanford.edu/~amaas//data/sentiment/}{here} and was already split 50/50 into train/test.
    \item \underline{TWTR} The Twitter Sentiment140 data set was downloaded from \href{http://cs.stanford.edu/people/alecmgo/trainingandtestdata.zip}{here} and was pre-processed according to the method highlighted \href{https://towardsdatascience.com/another-twitter-sentiment-analysis-bb5b01ebad90}{here}.
    \item \underline{MPQA} The MPQA opinion corpus can be found \href{http://mpqa.cs.pitt.edu/}{here} and was pre-processed using the code found \href{https://github.com/AcademiaSinicaNLPLab/sentiment_dataset}{here}.
    \item \underline{SST2} The Stanford Sentiment Treebank data set was downloaded from \href{https://nlp.stanford.edu/sentiment/}{here} and the pre-processing code can be found \href{https://github.com/AcademiaSinicaNLPLab/sentiment_dataset}{here}.
\end{enumerate}

\paragraph{Creating feature vectors} After all data sets were downloaded and pre-processed, the different feature vectors were constructed using \texttt{CounterVectorizer} and \texttt{TfidfVectorizer} from Sklearn \cite{scikit}. Counter vector, tf-idf, and tf-idf word bigrams use the \texttt{analyzer = `word'} specification while the tf-idf char bigrams use \texttt{analyzer = `char'}.

\paragraph{Two-stage procedures} For experiments 2 and 3, all standard models were trained in Sklearn \cite{scikit}. In particular, the following settings were used in stage 2 for each model
\begin{enumerate}
    \item  \texttt{LogisticRegression(penalty=`l2', solver=`lbfgs', C =1e4, max\_iter=1e2)}
    \item \texttt{LinearSVC(C = 1e4)}
    \item \texttt{MultinomialNB(alpha=a)}
\end{enumerate}
In the first stage of the two stage procedures, the following settings were used for each of the different feature selection methods
\begin{enumerate}
    \item \texttt{LogisticRegression(random\_state=0, C = $\lambda_1$,penalty=`l1',solver=`saga',$\backslash$}\\ \texttt{max\_iter=1e2)}
    \item \texttt{clf = LogisticRegression(C = 1e4, penalty=`l2', \
                    solver = `lbfgs', $\backslash$} \\
                    \texttt{max\_iter = 1e2).fit(train\_x,train\_y)} \\
        \texttt{selector\_log = RFE(clf, $k$), step=0.3)}
    \item \texttt{Lasso(alpha = $\lambda_2$, 
                    selection=`cyclic', tol = 1e-5)}
    \item \texttt{LinearSVC(C =$\lambda_3$, \
                    penalty=`l1',dual=False)}
    \item \texttt{clf = LinearSVC(C = 1e4, penalty=`l2',dual=False).fit(train\_x,train\_y)} \\
            \texttt{selector\_svm = RFE(clf,$k$, step=0.3)}
    \item \texttt{MultinomialNB(alpha=a)}
\end{enumerate}
where $\lambda_i$ are hyper-parameters used by the $\ell_1$ methods to achieve a desired sparsity level $k$. $a$ is a hyper-parameter for the different MNB models which we compute using cross validation. 

\paragraph{Hyper-parameters}
For each of the $\ell_1$ methods we manually do a grid search over all hyper-parameters to achieve an approximate desired sparsity pattern. For determining the hyper-parameter for the MNB models, we employ 10-fold cross validation on each data set for each type of feature vector and determine the best value of $a$. In total, this is $16 + 20 = 36$ values of $a$ -- $16$ for experiment 2 and $20$ for experiment 3. In experiment 2, we do not use the twitter data set since computing the $\lambda_i$'s to achieve a desired sparsity pattern for the $\ell_1$ based feature selection methods was computationally intractable.

\end{document}